\documentclass{article}

 \usepackage[preprint]{neurips_2026}

\usepackage[utf8]{inputenc} 
\usepackage[T1]{fontenc}    
\usepackage{hyperref}       
\usepackage{url}            
\usepackage{booktabs}       
\usepackage{amsfonts}       
\usepackage{nicefrac}       
\usepackage{microtype}      
\usepackage{xcolor}         

\usepackage{amsmath,amssymb,amsthm}
\usepackage{mathtools}
\usepackage{enumitem}
\usepackage{hyperref}
\usepackage{xcolor}
\usepackage{algorithm}
\usepackage{algpseudocode}
\usepackage{subcaption}
\usepackage{graphicx}
\usepackage{wrapfig}

\usepackage{array}
\usepackage{colortbl}
\usepackage{longtable}
 
\definecolor{headerblue}{RGB}{41, 65, 122}
\definecolor{lightgray}{RGB}{245, 245, 245}
\definecolor{correctgreen}{RGB}{34, 120, 60}
\definecolor{incorrectred}{RGB}{180, 30, 30}

\newtheorem{definition}{Definition}

\newtheorem{lemma}{Lemma}
\newtheorem{theorem}{Theorem}

\newtheorem{remark}{Remark}

\newtheorem{atheorem}{Theorem}[section]
\newtheorem{aproposition}{Proposition}[section]
\newtheorem{aremark}{Remark}[section]

\usepackage{relsize}
\newcommand{\std}[1]{{\color{gray}{\scriptsize $\pm$#1}}}

\title{Hidden States Know Where Reasoning Diverges: Credit Assignment via Span-Level Wasserstein Distance}

\author{%
\textbf{Xinzhu Chen$^{1}$\thanks{Equal contribution.}, Wei He$^{2}$\footnotemark[1], Huichuan Fan$^2$, Wenzhe Niu$^3$, Zhongxiang Sun$^4$, Xuanru Wang$^5$,} \\ 
\textbf{Jiuchong Gao$^2$, Jinghua Hao$^2$, Renqing He$^2$, Weijie Yu$^6$\thanks{Corresponding author.}} \\ 
$^1$Beijing University of Posts and Telecommunications,
$^2$Meituan, $^3$Tianjin University,\\ 
$^4$Renmin University of China, $^5$The University of Melbourne,\\ 
$^6$University of International Business and Economics\\ 
\texttt{c1456355244@gmail.com}, \texttt{yu@uibe.edu.cn}\\
}

\begin{document}

\maketitle

\begin{abstract}
Group Relative Policy Optimization (GRPO) performs coarse-grained credit assignment in reinforcement learning with verifiable rewards (RLVR) by assigning the same advantage to all tokens in a rollout. Process reward models can provide finer-grained supervision, but they require step-level annotation or additional reward modeling.
We show that hidden-state distributions contain a useful signal for local reasoning quality that can be extracted using only outcome-level correctness labels available in RLVR. 
Specifically, within each GRPO group, the Wasserstein distance between span-level hidden state distributions of correct and incorrect rollouts increases around regions where their local reasoning quality diverges. 
This association holds both across examples and within individual trajectories, suggesting that hidden-state distributional divergence can serve as a self-supervision signal for fine-grained credit assignment.
We formalize this observation with a separation theorem showing that, under mild structural assumptions, post-divergence spans have larger Wasserstein distances than pre-divergence spans whenever the population-level distributional gap exceeds finite-sample noise. 
Motivated by this result, we propose \textbf{S}pan-level \textbf{H}idden state \textbf{E}nabled \textbf{A}dvantage \textbf{R}eweighting (SHEAR), which modifies GRPO by using span-level Wasserstein distances to scale token-level advantages, amplifying updates on tokens whose hidden states are more separated from the opposing group.
The method requires no additional model and only minimal changes to the training pipeline.
Experiments on five mathematical reasoning benchmarks and five code generation
benchmarks show improvements over standard GRPO and strong performance relative to supervised process reward models, while requiring no additional annotation or reward model training.

\end{abstract}

\section{Introduction}\label{sec:intro}
Reinforcement learning with verifiable rewards (RLVR) is a widely used paradigm for improving reasoning in large language models~\cite{yan2025rlvr_2,zhao2025rlvr_1}, with Group Relative Policy Optimization (GRPO)~\cite{shao2024grpo} being one of its most commonly adopted algorithms. In GRPO, the model samples multiple reasoning trajectories for each problem, receives outcome-level correctness feedback, and updates its policy based on the relative quality of these trajectories. A central challenge in this setting is credit assignment, that is, identifying which parts of a reasoning chain are responsible for success or failure. GRPO assigns the same advantage to every token in a rollout, leaving this finer-grained problem unresolved. This limitation may become more severe on the complicated tasks, which require more thinking and execution steps~\cite{lightman2023math500}. A single erroneous step can derail an otherwise sound trajectory, yet all tokens receive the same gradient signal regardless of their contribution to the final outcome. This is especially problematic in long-form mathematical reasoning\cite{ai_mo_aime2024,math_ai_amc23,opencompass_aime2025} and code generation~\cite{he2024OlympiadBench,jain2024livecodebench,liu2023Humaneval_mbpp}, where chains routinely span hundreds of tokens.

The dominant approach to fine-grained credit assignment is the \emph{process reward model} (PRM)~\cite{cheng2025pure,cui2025prime,lightman2023math500}, which assigns rewards to individual reasoning steps. PRMs have shown clear advantages over outcome-level supervision, but they rely on costly step-level supervision. Training a PRM typically requires either human-annotated intermediate labels or an auxiliary reward model trained on such annotations, making the resulting credit signal dependent on the quality of the supervision. 
This issue is particularly acute on complicated tasks, where the correctness of intermediate steps is often ambiguous even to experts~\cite{cui2025prime}. These considerations motivate a natural question. Can we obtain process-level credit signals without step-level annotation or an additional reward model?

We approach this question from a representational perspective. Instead of training an external model to score individual reasoning steps, we examine whether the language model's own hidden states contain a signal correlated with local reasoning quality. Hidden states summarize the evolving reasoning context, so reasoning errors may induce systematic distributional differences between correct and incorrect trajectories. If such differences are detectable at the span level, they could provide a self-supervised signal for process-level credit assignment using only the outcome-level correctness labels available in RLVR. We investigate this possibility by measuring the Wasserstein distance~\cite{arjovsky2017wasserstein,peyre2019computational} between span-level hidden state distributions of correct and incorrect rollouts within each GRPO group.

Our empirical analysis suggests that such a signal indeed exists. By comparing span-level hidden state distributions between correct and incorrect rollouts within the same GRPO group, we find that the Wasserstein distance provides an effective proxy for local reasoning quality. It remains low where correct and incorrect trajectories are distributionally similar and increases where they diverge, with the two trends exhibiting aligned transition zones (Section~\ref{sec:observation}). Building on this finding, we propose a simple modification to GRPO in which span-level Wasserstein distances are used as token-level weights in the advantage function, amplifying the existing advantage direction at positions where correct and incorrect trajectories are most distinct. The method requires no additional model, no step-level labels, and only minimal changes to the existing training pipeline (Section~\ref{sec:method}). We further formalize the underlying phenomenon with a separation theorem showing that, under mild structural assumptions, post-divergence spans have larger Wasserstein distances than pre-divergence spans whenever the population-level distributional difference exceeds finite-sample noise (Section~\ref{sec:theory}).

Our contributions are three fold:
\begin{enumerate}[nosep, leftmargin=1.5em]
    \item \textbf{Empirical finding.} We show that span-level Wasserstein distance between hidden state distributions of correct and incorrect rollouts serves as a self-supervised signal that tracks reasoning quality without requiring any step-level annotation. This finding suggests that the model's own representations carry process-level information that can be exploited without training a separate reward model (Section~\ref{sec:observation}).
    
    \item \textbf{Theoretical foundation.} We prove a separation theorem establishing that post-error spans exhibit larger Wasserstein distances than pre-error spans whenever the population-level distributional difference exceeds finite-sample noise, which is a condition we verify empirically. We further extend this result to the group-level minimum distance used in practice and provide a principled justification for using Wasserstein distance as a credit assignment signal (Section~\ref{sec:theory}).
    
    \item \textbf{Method and empirical validation.} We propose \textbf{S}pan-level \textbf{H}idden state \textbf{E}nabled \textbf{A}dvantage \textbf{R}eweighting (SHEAR),
    a simple modification to GRPO that uses span-level Wasserstein distances to weight token-level advantages. Experiments on five mathematical reasoning benchmarks and five code generation benchmarks show consistent improvements over standard GRPO and strong performance relative to supervised process reward models, while requiring no additional annotation or reward model training (Section~\ref{sec:experiments}).
\end{enumerate}


\section{Distributional Divergence in Hidden States}\label{sec:observation}
 
Before presenting our method, we establish an empirical observation that motivates and grounds the entire approach--- in a controlled diagnostic setting, hidden-state distributional structure is strongly associated with local reasoning quality, and Wasserstein distance provides a useful way to measure this structure.
 
\paragraph{Setup.}
We use MATH500~\cite{lightman2023math500} as the evaluation benchmark with Qwen2.5-Math-7B~\cite{yang2024qwen2.5} as the base model. For each problem, we sample $G{=}8$ rollouts and retain only problems where the model produces a mix of correct and incorrect answers. This filtering is required because our diagnostic compares correct and incorrect trajectories within the same rollout group and it excludes uniformly easy or uniformly hard problems where no opposing group is available. The retained subset comprises $83.8\%$ of MATH500 problems, covering a range of difficulty levels (Please refer to Appendix~\ref{app:acc_distribution} for details).
 
For each retained rollout, we extract prefixes at varying truncation points using a sliding window (starting length $w{=}100$ tokens, stride $s{=}25$ tokens), excluding spans that contain the final answer to avoid trivial leakage from answer-revealing tokens. At each truncation point, we estimate a \emph{continuation success rate} by letting the model complete the reasoning 16 times from that prefix and recording the fraction of correct completions. We use this quantity as a proxy for local reasoning quality where prefixes that preserve a valid reasoning trajectory should admit a higher probability of correct completion.
Independently, at each prefix position we collect the hidden states of the most recent $w$ tokens and compute the Wasserstein distance between this local hidden state distribution and the closest matching distribution from the opposing group (correct vs.\ incorrect). We adopt Wasserstein distance as our primary metric due to its ability to capture distributional differences beyond mean shift. The comparison with alternative metrics is provided in Section~\ref{sec:exp-ablation}. Full details of the distance computation are described in Section~\ref{sec:algorithm}.

\begin{figure}[t]
\centering
\begin{subfigure}[b]{0.48\textwidth}
    \includegraphics[width=\textwidth]{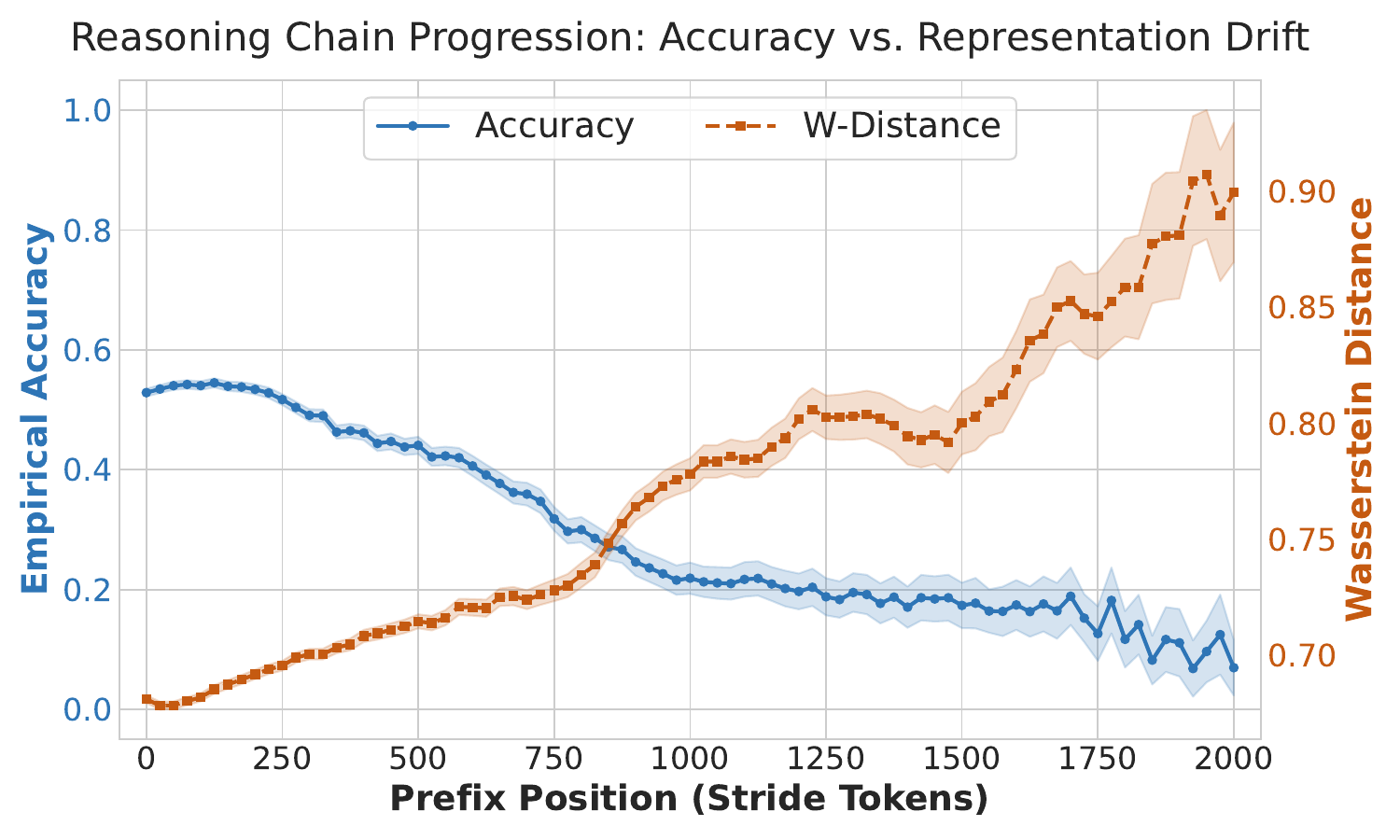}
    \caption{Aggregate trend along the reasoning chain}
    \label{fig:aggregate_trend}
\end{subfigure}
\hfill
\begin{subfigure}[b]{0.48\textwidth}
    \includegraphics[width=\textwidth]{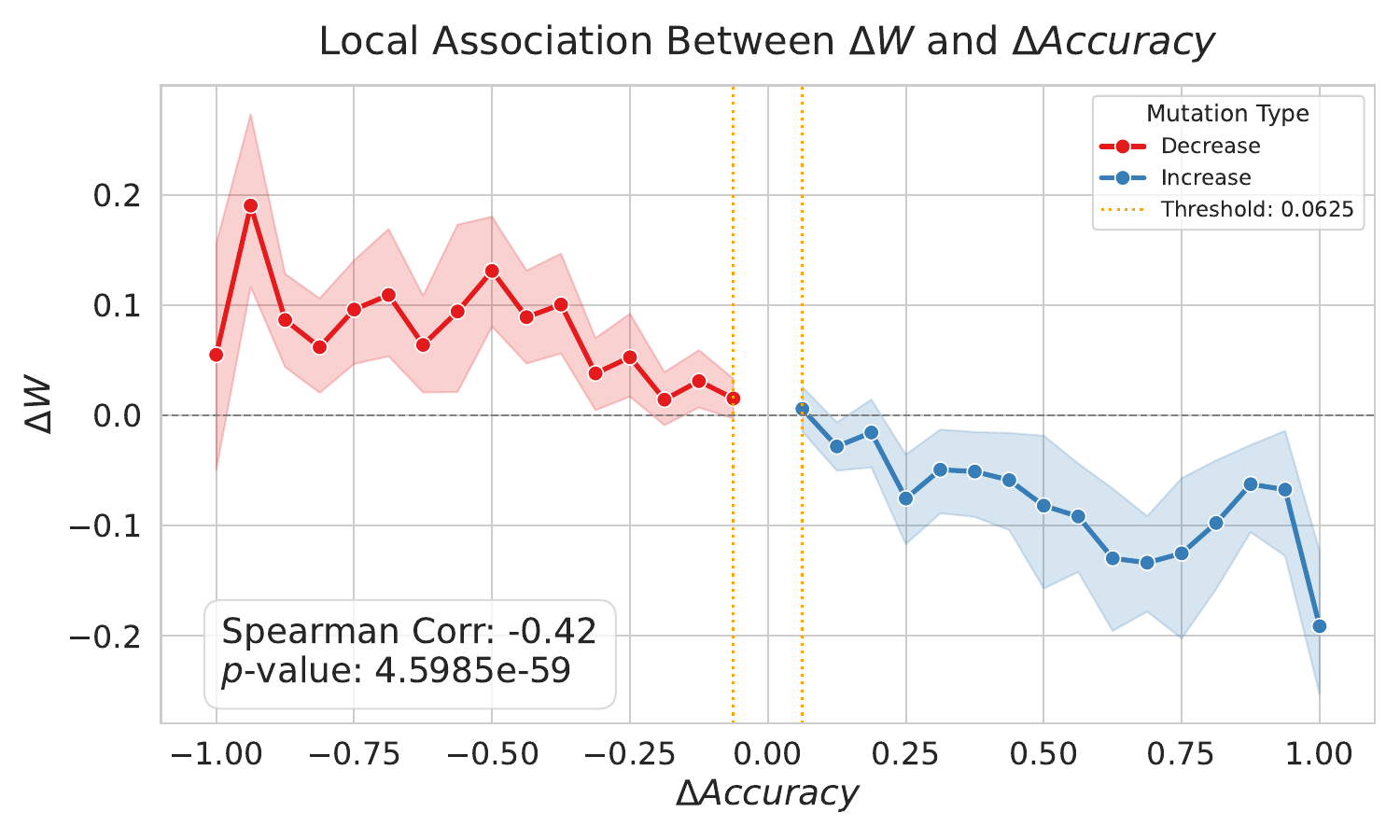}
    \caption{Stride-wise local association} 
    \label{fig:local_association}
\end{subfigure}
\caption{\textbf{Hidden state divergence tracks reasoning quality 
at both aggregate and local levels.} (a)~As reasoning progresses, 
continuation success (blue, left axis) declines while Wasserstein 
distance to the opposing group (brown, right axis) rises, with 
closely aligned transition zones with Spearman's $\rho = -0.96$. (b)~At positions where continuation success changes by 
at least one completion step ($|\Delta\text{Accuracy}| \geq 
0.0625$), the concurrent change in Wasserstein distance is 
negatively correlated (Spearman's $\rho = -0.42$). This within-trajectory association 
confirms that the signal in~(a) reflects genuine local 
representational dynamics. Results are on MATH500 with Qwen2.5-Math-7B.}
\label{fig:observation}
\end{figure}

\paragraph{Aggregate trend.}
Figure~\ref{fig:aggregate_trend} plots both quantities as a function of prefix position, aggregated across all retained problems. As prefixes extend deeper into the reasoning chain, continuation success tends to decline while Wasserstein distance tends to increase, and the transition zones of the two curves are closely aligned. The region where continuation success begins to decline coincides with the region where Wasserstein distance begins to rise.

\paragraph{Local association.}
To examine whether the association holds at the local level, we analyze stride-wise changes in both quantities within each trajectory. For each pair of consecutive prefix positions within the same rollout, we compute the change in continuation success ($\Delta\text{Accuracy}$) and the concurrent change in Wasserstein distance ($\Delta W$). Since each continuation success estimate is based on 16 completions, the minimum detectable change is $1/16 = 0.0625$; we use this as the threshold for identifying meaningful local shifts, retaining only positions where $|\Delta\text{Accuracy}| \geq 0.0625$.
Figure~\ref{fig:local_association} plots $\Delta W$ against $\Delta\text{Accuracy}$ for these local shift events. We can observe that when continuation success locally decreases (negative $\Delta\text{Accuracy}$), Wasserstein distance tends to locally increase (positive $\Delta W$), and vice versa. 
This within-trajectory association suggests that the distributional divergence signal contains a genuine local component, such that local deterioration in reasoning is accompanied by local divergence of hidden states from the opposing group.

\paragraph{Implications.}
These two analyses provide complementary evidence that hidden state distributions carry a signal associated with reasoning quality. The aggregate trend (Figure~\ref{fig:aggregate_trend}) shows that Wasserstein distance and continuation success move in opposite directions along the reasoning chain. The local analysis (Figure~\ref{fig:local_association}) further shows that this association is not solely an aggregate phenomenon but reflects within-trajectory dynamics. At positions where continuation success locally shifts, Wasserstein distance shifts concurrently in the opposite direction. If this distributional divergence signal can be reliably detected, it could serve as a self-supervised proxy for process-level credit without any step-level annotation.

This observation raises two questions that the remainder of the paper addresses. First, can we exploit this signal to improve credit assignment during training? We describe our method in Section~\ref{sec:method}. Second, under what conditions is this 
distributional separation theoretically guaranteed? We formalize a separation theorem in Section~\ref{sec:theory} and validate the full approach empirically in Section~\ref{sec:experiments}.

\section{Method}\label{sec:method}
\subsection{Preliminaries: GRPO}\label{sec:grpo}

We consider reinforcement learning with verifiable rewards (RLVR) for reasoning tasks. Given an input $x$ (e.g., a math problem), a policy $\pi_\theta$ generates a rollout $y^{(i)} = (y_1^{(i)}, \ldots, y_{T_i}^{(i)})$ and receives a binary reward $r^{(i)} \in \{0, 1\}$ indicating correctness. In GRPO, a group of $G$ rollouts $\mathcal{G}(x) = \{y^{(1)}, \ldots, y^{(G)}\}$ is sampled for each input, and the policy gradient takes the form
\begin{equation}\label{eq:grpo}
    \nabla_\theta J = \mathbb{E}_{y^{(i)} \sim \pi_\theta} \left[\sum_{t=1}^{T_i} A^{(i)} \nabla_\theta \log \pi_\theta(y_t^{(i)} \mid y_{<t}^{(i)}, x)\right],
\end{equation}
where $A^{(i)} = \frac{r^{(i)} - \bar{r}_{\mathcal{G}}}{\sigma_{\mathcal{G}} + \epsilon}$ is the group-normalized advantage, with $\bar{r}_{\mathcal{G}}$ and $\sigma_{\mathcal{G}}$ denoting the mean and standard deviation of rewards within the group.

A fundamental limitation of Eq.~\ref{eq:grpo} is that every token in a rollout receives the \emph{same} advantage $A^{(i)}$. For example, if an incorrect rollout contains mostly correct reasoning but deviates at some step $\tau$, then tokens before $\tau$ are penalized in the same way as those after $\tau$.

\subsection{Span-Level Wasserstein Distance}\label{sec:span-w}

Our method builds on the observation from Section~\ref{sec:observation} that correctness-relevant local differences induce detectable shifts in the model's hidden-state distributions. We measure these shifts using the Wasserstein distance between local spans of hidden states.

\paragraph{Span decomposition.} Given a rollout $y^{(i)}$ of length $T_i$, we partition it into overlapping spans using a sliding window of size $w$ and stride $s$ ($s \leq w$):
\begin{equation}
    \mathcal{S}^{(i)} = \left\{ S_k^{(i)} = [a_k^{(i)}, b_k^{(i)}] \;\middle|\; a_k^{(i)} = 1 + (k-1)s,\;\; b_k^{(i)} = \min(a_k^{(i)} + w - 1, T_i) \right\}.
\end{equation}
Each span $S_k^{(i)}$ defines an empirical distribution over hidden states
$\widehat{P}_k^{(i)} = \frac{1}{|S_k^{(i)}|} \sum_{t \in S_k^{(i)}} \delta_{h_t^{(i)}}$, 
where $h_t^{(i)} \in \mathbb{R}^d$ is the hidden state at the last transformer layer before LM-head for token $y_t^{(i)}$
~\cite{yu2026latentSurvey,zhang2025Relax}, and $\delta_{h_t^{(i)}}$ denotes the point mass at $h_t^{(i)}$.

\paragraph{Wasserstein distance.} The 1-Wasserstein distance between two distributions $P$ and $Q$ on $\mathbb{R}^d$ is
\begin{equation}\label{eq:w1}
    W_1(P, Q) = \inf_{\gamma \in \Gamma(P, Q)} \int \|x - y\|_2 \, d\gamma(x, y),
\end{equation}
where $\Gamma(P, Q)$ is the set of all couplings with marginals $P$ and $Q$. Intuitively, $W_1$ measures the minimum cost of transporting one distribution into the other. As observed in Section~\ref{sec:observation}, correctness-relevant local differences are reflected not only in shifts of hidden-state means but also in broader changes in span-level distributional shape. We therefore adopt $W_1$ as a distribution-level discrepancy rather than a point-estimate similarity such as cosine distance between span means (see Section~\ref{sec:theory} for theoretical justification and Section~\ref{sec:exp-ablation} for empirical comparison with alternative metrics).

In practice, we compute the entropically regularized approximation $W_\epsilon$ via the Sinkhorn algorithm~\cite{cuturi2013sinkhorn}, which is efficient to compute on GPUs. The approximation satisfies $W_1 \leq W_\epsilon \leq W_1 + \epsilon \log n$ for discrete measures on $n$ points (Appendix~\ref{app:sinkhorn}).
\subsection{Wasserstein-Weighted Token-Level Advantage}\label{sec:algorithm}
\begin{figure}[t]
    \centering
    \includegraphics[width=0.95\linewidth]{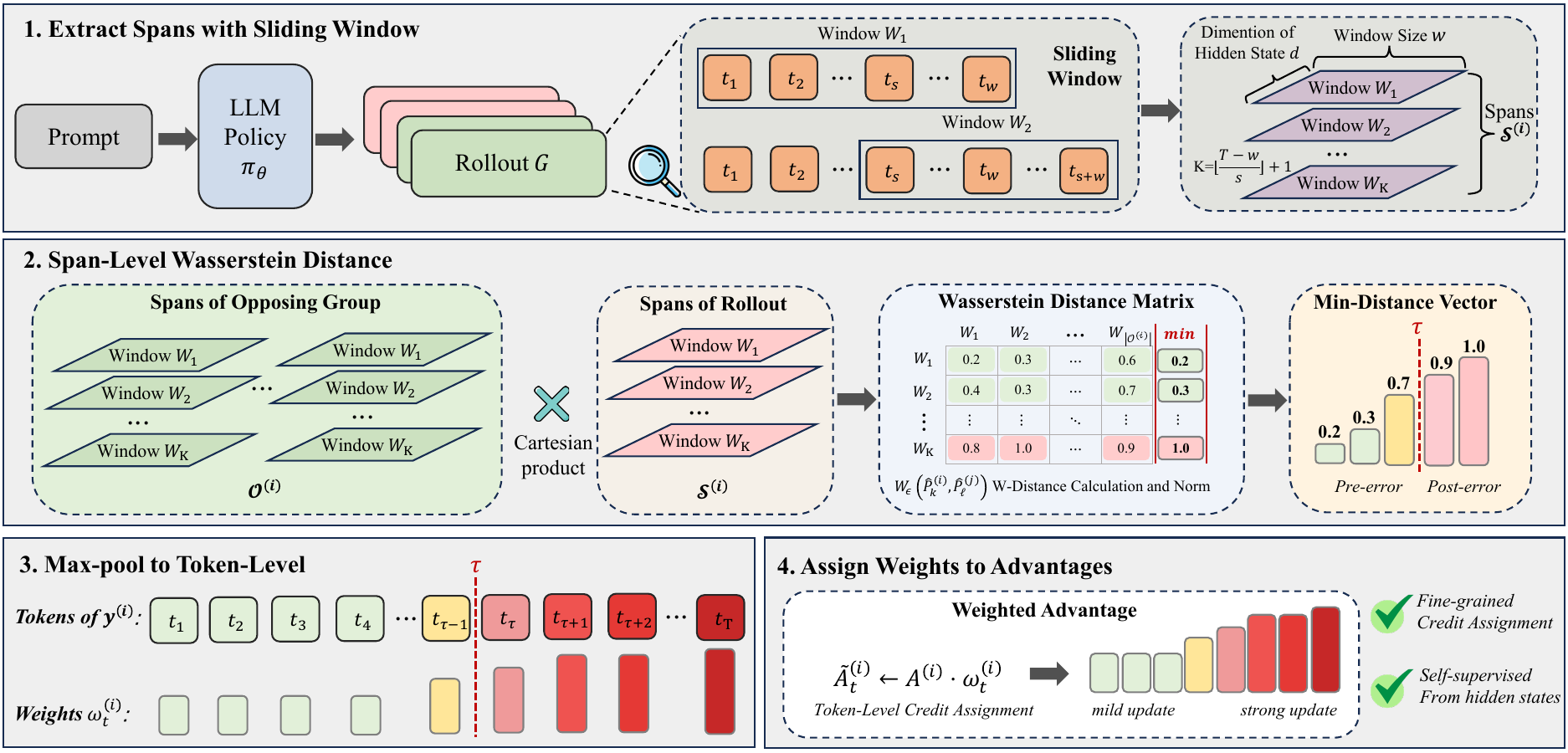}
    \caption{The overview of SHEAR. For each rollout group, we partition each trajectory into overlapping spans, compute the minimum Wasserstein distance from each span to the opposing rollout set, aggregate span-level discrepancies back to token-level weights through max-pooling over overlapping spans, and use these weights to rescale the rollout-level GRPO advantage at each token.}
    \label{fig:main_fig}
\end{figure}

Conceptually, our proposed SHEAR first extracts a process-level signal from span-level hidden-state divergence and then uses this signal to rescale token-level policy-gradient updates in GRPO. We use spans to estimate local discrepancy relative to the opposing rollout set, and then project this span-level signal back to individual tokens through overlapping-window aggregation.

We partition the rollout group into correct rollouts $\mathcal{G}^+$ and incorrect rollouts $\mathcal{G}^-$. For each rollout $y^{(i)}$, we define its opposing set as
\begin{equation}
    \mathcal{O}^{(i)} =
    \begin{cases}
        \mathcal{G}^- , & r^{(i)} = 1,\\
        \mathcal{G}^+ , & r^{(i)} = 0.
    \end{cases}
\end{equation}
If $\mathcal{O}^{(i)}$ is empty, we fall back to the standard GRPO update for that rollout. The procedure of SHEAR in Algorithm~\ref{alg:method} has four steps.
\begin{algorithm}[t]
\caption{SHEAR}
\label{alg:method}
\begin{algorithmic}[1]
\Require Rollout group $\mathcal{G}(x)$, window size $w$, stride $s$, Sinkhorn parameter $\epsilon$
\For{each rollout $y^{(i)} \in \mathcal{G}(x)$}
    \If{$\mathcal{O}^{(i)} \neq \emptyset$}
        \For{each span $S_k^{(i)} \in \mathcal{S}^{(i)}$}
            \State $d_k^{(i)} \gets \min\limits_{y^{(j)} \in \mathcal{O}^{(i)}} \; \min\limits_{S_\ell^{(j)} \in \mathcal{S}^{(j)}} W_\epsilon\!\left(\widehat{P}_k^{(i)}, \widehat{P}_\ell^{(j)}\right)$\Comment{Min distance to opposing group}
        \EndFor
    \EndIf
\EndFor
\State $N_{\mathcal{G}} \gets \sum_{i=1}^{G} T_i$
\State $\bar{n} \gets \frac{1}{N_{\mathcal{G}}}\sum_{i=1}^{G}\sum_{t=1}^{T_i} \|h_t^{(i)}\|_2$ \Comment{Global mean norm over all hidden states in group}
\For{each rollout $y^{(i)} \in \mathcal{G}(x)$}
    \For{each token position $t = 1, \ldots, T_i$}
        \If{$\mathcal{O}^{(i)} \neq \emptyset$}
            \State $\omega_t^{(i)} \gets \frac{1}{\bar{n}} \max\limits_{k : t \in S_k^{(i)}} d_k^{(i)}$ \Comment{Max-pool \& normalize}
        \Else
            \State $\omega_t^{(i)} \gets 1$
        \EndIf
        \State $\widetilde{A}_t^{(i)} \gets A^{(i)} \cdot \omega_t^{(i)}$ \Comment{Weighted advantage}
    \EndFor
\EndFor
\end{algorithmic}
\end{algorithm}

\textbf{Step 1: Span-level minimum distance.} For each span $S_k^{(i)}$ in rollout $y^{(i)}$, we compute its minimum Sinkhorn distance to all spans in the opposing set $\mathcal{O}^{(i)}$ (line 4 in Algorithm~\ref{alg:method}). This minimum makes the signal conservative with respect to multiple valid solution paths so that a span receives a large discrepancy only if it remains dissimilar to all candidate alternatives in the opposing set. We compare this choice against alternative aggregation rules in Section~\ref{sec:exp-ablation}.

\textbf{Step 2: Global norm normalization.} To stabilize magnitudes across varying latent scales, we divide all span-level discrepancies by the global mean hidden-state norm
$\bar{n} = \frac{1}{N_{\mathcal{G}}}\sum_{i=1}^{G}\sum_{t=1}^{T_i} \|h_t^{(i)}\|_2$, 
where $N_{\mathcal{G}} = \sum_{i=1}^{G} T_i$ is the total number of tokens in the rollout group. Since $\bar{n}$ is a positive constant shared across the entire group, this rescaling preserves the ordering of span discrepancies within the group, and the ranking-based separation established for $W_1$ in Section~\ref{sec:theory} transfers to the normalized weights. A side effect is that different rollouts may receive different total gradient magnitudes, introducing a form of cross-rollout reweighting. We isolate this effect in the ablation study (Section~\ref{sec:exp-ablation}).

\textbf{Step 3: Token-level aggregation via max-pooling.} Considering that each token may belong to multiple overlapping spans, we assign token $t$ in rollout $y^{(i)}$ the maximum normalized discrepancy among all spans that cover it (line 10 in Algorithm~\ref{alg:method}), i.e., $\omega_t^{(i)} = \frac{1}{\bar{n}} \max_{k : t \in S_k^{(i)}} d_k^{(i)}$. 
This choice emphasizes the strongest local discrepancy signal associated with each token, which is especially useful in local transition regions where overlapping spans may mix low- and high-discrepancy patterns. We compare max-pooling against mean-pooling in Appendix~\ref{app:design}.

\textbf{Step 4: Weighted advantage.} The token-level weighted advantage is obtained by multiplying the rollout-level GRPO advantage by the token weight
$    \widetilde{A}_t^{(i)} = A^{(i)} \cdot \omega_t^{(i)}$. 
For incorrect rollouts ($A^{(i)} < 0$), tokens in high-discrepancy regions receive larger negative updates. For correct rollouts ($A^{(i)} > 0$), tokens whose hidden-state distributions are more strongly separated from the opposing set receive larger positive updates.

\section{Theoretical Analysis}\label{sec:theory}

The observations in Section~\ref{sec:observation} suggest that span-level Wasserstein distances align empirically with local reasoning quality. In this section, we ask under what conditions such a separation can be expected theoretically. Our goal is not to model the full complexity of real reasoning trajectories, but to isolate a minimal stylized setting in which post-divergence spans become distinguishable from pre-divergence spans once the population-level distributional signal exceeds finite-sample noise. Under this view, the theory provides conditional support for the weighting mechanism used in Section~\ref{sec:algorithm}. In particular, the global normalization by $\bar{n}$ preserves the ordering of span-level discrepancies within a rollout group, so the separation established below is consistent with the intended weighting pattern.

For readability, the theorems are stated for the exact Wasserstein distance $W_1$, whereas the algorithm uses the entropically regularized approximation $W_\epsilon$. As discussed in Appendix~\ref{app:sinkhorn}, the approximation satisfies$W_1 \leq W_\epsilon \leq W_1 + \epsilon \log n$, for discrete measures on $n$ points. 
In our experiments, this approximation gap is small relative to the empirical separation margins observed between pre- and post-divergence spans, provided the separation gap exceeds $\epsilon \log n$.

\paragraph{Assumptions.}
We consider a correct rollout $y^+$ and an incorrect rollout $y^-$ generated for the same input. Let $h_t^+, h_t^- \in \mathbb{R}^d$ denote their token-level hidden states. The analysis is based on the following stylized assumptions.

\emph{(A1) Divergence structure.}
There exists a divergence point $\tau$ such that, for $t \leq \tau$, the hidden states $h_t^+$ and $h_t^-$ are independent draws from the same marginal distribution $P_t$, while for $t > \tau$, the correct rollout remains distributed as $P_t$ and the incorrect rollout is distributed as $Q_t \neq P_t$. This assumption requires distributional equivalence rather than token-level alignment: two rollouts may differ in surface form while still sharing the same hidden-state distribution at the same reasoning stage. As with any stylized model, exact distributional equivalence is an idealization. In practice a small residual gap may persist before $\tau$. Figure~\ref{fig:aggregate_trend} provides empirical support for treating this gap as negligible where the mean Wasserstein distance in the first quartile of prefix positions ($\approx 0.67$) is smaller than in the last quartile ($\approx 0.75$), and the separation theorems below remain valid whenever this residual pre-divergence gap is dominated by the post-divergence signal $D(S)$---a condition verified empirically in Section~\ref{sec:exp:1}.

\emph{(A2) Bounded support.}
For all positions $t$, the hidden states satisfy $\|h_t\|_2 \leq M$.
This is satisfied in practice by modern transformer architectures, which apply LayerNorm~\cite{lei2016layernorm} or RMSNorm~\cite{zhang2019RMSnorm} at each layer, bounding the magnitude of hidden state vectors.

\emph{(A3) Finite-sample concentration.}
For independent samples $X_t \sim P_t$ with $\|X_t\|_2 \leq M$ (in this assumption, $P_t$ denotes a generic bounded-support distribution and is not restricted to the correct-rollout marginals in (A1), in particular, the bound applies equally when $P_t$ is replaced by $Q_t$), let

\[
\widehat{P}_n = \frac{1}{n}\sum_{t=1}^n \delta_{X_t},
\qquad
\bar{P}_n = \frac{1}{n}\sum_{t=1}^n P_t .
\]
We assume
\[
\mathbb{E}\!\left[ W_1(\widehat{P}_n, \bar{P}_n) \right]
\leq \tilde{C}_d M n^{-1/d}.
\]
For i.i.d.\ samples, bounds of this form are classical~\cite{fournier2015rate}. Our setting is independent but not necessarily identically distributed, so we adopt this rate as an assumption capturing the finite-sample noise floor relevant to our analysis and Appendix~\ref{app:concentration} discusses why this assumption is plausible in our setting.

We define the corresponding noise floor
\begin{equation}
    \eta(n,d) := 2\tilde{C}_d M n^{-1/d},
\end{equation}
which bounds the expected Wasserstein distance between two empirical measures drawn from the same span-level population mixture up to sampling variability. The practical role of (A3) is to separate finite-sample noise from population-level signal. Whether $\eta(n,d)$ is small relative to the observed discrepancy is evaluated empirically in Section~\ref{sec:exp:1}.

\paragraph{Pairwise separation.}
Let $S = [a,b]$ be a span of length $n=b-a+1$. Define the corresponding span-level population mixtures
\begin{equation}
    \bar{P}_S^+ := \frac{1}{n}\sum_{t=a}^{b} P_t,
    \qquad
    \bar{P}_S^- := \frac{1}{n}\sum_{t=a}^{b} Q_t,
\end{equation}
where, in the pre-divergence regime, $\bar{P}_S^- = \bar{P}_S^+$. We define the population distance
\begin{equation}\label{eq:pop-distance}
    D(S) := W_1(\bar{P}_S^-, \bar{P}_S^+).
\end{equation}
Let $\widehat{P}_S^-$ and $\widehat{P}_S^+$ denote the corresponding empirical span distributions.

\begin{theorem}[Pairwise Separation]\label{thm:separation}
Under assumptions (A1)--(A3), for a span $S=[a,b]$ of length $n$ in $y^-$ paired with the corresponding span in $y^+$, the following hold:
\begin{enumerate}[label=(\roman*), nosep]
    \item \textbf{Pre-divergence} ($b \leq \tau$):
    \[
    \mathbb{E}\!\left[W_1(\widehat{P}_S^-, \widehat{P}_S^+)\right] \leq \eta(n,d).
    \]
    
    \item \textbf{Post-divergence} ($a > \tau$):
    \[
    \mathbb{E}\!\left[W_1(\widehat{P}_S^-, \widehat{P}_S^+)\right] \geq D(S) - \eta(n,d).
    \]
    
    \item \textbf{Strict separation.} If $D(S) > 2\eta(n,d)$, then the expected post-divergence distance is strictly larger than the expected pre-divergence distance.
\end{enumerate}
\end{theorem}

\begin{proof}[Proof sketch]
For (i), both empirical measures are drawn from the same span-level population mixture, so by the triangle inequality and assumption (A3),
\[
\mathbb{E}\!\left[W_1(\widehat{P}_S^-, \widehat{P}_S^+)\right] \leq \eta(n,d).
\]
For (ii), the reverse triangle inequality gives
\[
W_1(\widehat{P}_S^-, \widehat{P}_S^+) \geq
W_1(\bar{P}_S^-, \bar{P}_S^+) -
W_1(\widehat{P}_S^-, \bar{P}_S^-) -
W_1(\widehat{P}_S^+, \bar{P}_S^+).
\]
Taking expectations and applying (A3) to the two deviation terms yields
\[
\mathbb{E}\!\left[W_1(\widehat{P}_S^-, \widehat{P}_S^+)\right]
\geq D(S) - \eta(n,d).
\]
Part (iii) follows by subtracting the pre-divergence bound from the post-divergence lower bound. Full proofs are given in Appendix~\ref{app:proof-separation}.
\end{proof}

The quantity $D(S)$ captures distributional differences beyond first-moment shifts. By the Kantorovich--Rubinstein duality,
\[
D(S)
=
\sup_{\|f\|_{\mathrm{Lip}}\leq 1}
\left|
\frac{1}{n}\sum_{t=a}^{b}
\left(
\mathbb{E}_{Q_t}[f] - \mathbb{E}_{P_t}[f]
\right)
\right|.
\]
Restricting $f$ to linear functions recovers first-moment differences, whereas the supremum over all 1-Lipschitz functions also captures changes in variance, shape, and other aspects of span-level distributional structure. In this sense, Wasserstein distance is strictly more general than first-moment comparisons based only on span means (Appendix~\ref{app:why-w}; see also Section~\ref{sec:exp-ablation}).

\paragraph{Group-level separation.}
Algorithm~\ref{alg:method} uses the minimum span-level discrepancy to the opposing set rather than a single paired comparison. To analyze this construction, consider an incorrect rollout $y^-$ and a set of correct rollouts $\{y^{+,1}, \ldots, y^{+,K^+}\}$. For each correct rollout $y^{+,k}$, let $\tau_k$ denote its divergence point relative to $y^-$. Define the latest divergence point
\[
\tau^* := \max_k \tau_k,
\]
that is, the largest position up to which $y^-$ remains distributionally aligned with at least one correct rollout.

For a span $S$ in $y^-$, define the group-level empirical minimum discrepancy
\begin{equation}
    d_{\min}(S)
    :=
    \min_{k,j}
    W_1\!\left(\widehat{P}_S^-, \widehat{P}_{S_j}^{+,k}\right),
\end{equation}
where $\widehat{P}_{S_j}^{+,k}$ denotes the empirical distribution of span $S_j$ from correct rollout $y^{+,k}$. Define the corresponding population quantity
\begin{equation}
    D^*(S)
    :=
    \min_{k,j}
    W_1\!\left(\bar{P}_S^-, \bar{P}_{S_j}^{+,k}\right).
\end{equation}

\begin{theorem}[Group-Level Separation]\label{thm:group}
Under assumptions (A1)--(A3), the following hold:
\begin{enumerate}[label=(\roman*), nosep]
    \item If $b \leq \tau^*$, then
    \[
    \mathbb{E}[d_{\min}(S)] \leq \eta(n,d).
    \]
    
    \item If $a > \tau^*$, then
    \[
    \mathbb{E}[d_{\min}(S)] \geq D^*(S) - \eta(n,d).
    \]
    
    \item If $D^*(S) > 2\eta(n,d)$, then the expected post-divergence minimum discrepancy is strictly larger than the expected pre-divergence minimum discrepancy.
\end{enumerate}
\end{theorem}

\begin{proof}[Proof sketch]
If $b \leq \tau^*$, then there exists at least one correct rollout with a matching pre-divergence span, so the minimum discrepancy admits an expected upper bound of $\eta(n,d)$ by the pairwise argument. If $a > \tau^*$, then for every opposing span the reverse triangle inequality implies a lower bound of the form
\[
W_1(\widehat{P}_S^-, \widehat{P}_{S_j}^{+,k})
\geq
W_1(\bar{P}_S^-, \bar{P}_{S_j}^{+,k}) - \eta(n,d),
\]
and taking the minimum over all opposing spans yields the stated result. Full proofs are provided in Appendix~\ref{app:group}.
\end{proof}

The condition $D^*(S) > 2\eta(n,d)$ asks whether a post-divergence span is distributionally distinguishable from every opposing span up to finite-sample noise. A small $D^*(S)$ does not indicate failure of the method; rather, it means that the span is not globally distinctive from correct reasoning in the model's representation space, in which case assigning a low weight is appropriate. We do not claim that all post-divergence spans are separable. The theorem is explicitly conditional on the population-level separation exceeding the finite-sample noise floor. Extensions to multiple divergence points and the symmetric analysis for correct rollouts are provided in Appendices~\ref{app:multi-div} and~\ref{app:symmetric}.

\section{Experiments}

\label{sec:experiments}
\subsection{Analysis of the Wasserstein Signal}
\label{sec:exp:1}

\paragraph{Setup.}We adopt the same prefix-completion protocol as in Section~\ref{sec:observation} (MATH500, $G=8$ rollouts from Qwen2.5-Math-7B, sliding window with $w=100$, $s=25$, final-answer spans excluded), with one key difference: whereas Section ~\ref{sec:observation} analyzed the correlation between Wasserstein distance and empirical accuracy as a motivating observation, here we additionally leverage the ground-truth correctness labels of the original rollouts to directly evaluate whether Wasserstein distance can discriminate between correct and incorrect reasoning—the property that the separation theorem predicts. This yields $23,704$ spans from originally correct rollouts and $21,352$ from originally incorrect ones.
\begin{figure}[t]
    \centering
    \begin{subfigure}{0.48\linewidth}
        \centering
        \includegraphics[width=1\linewidth]{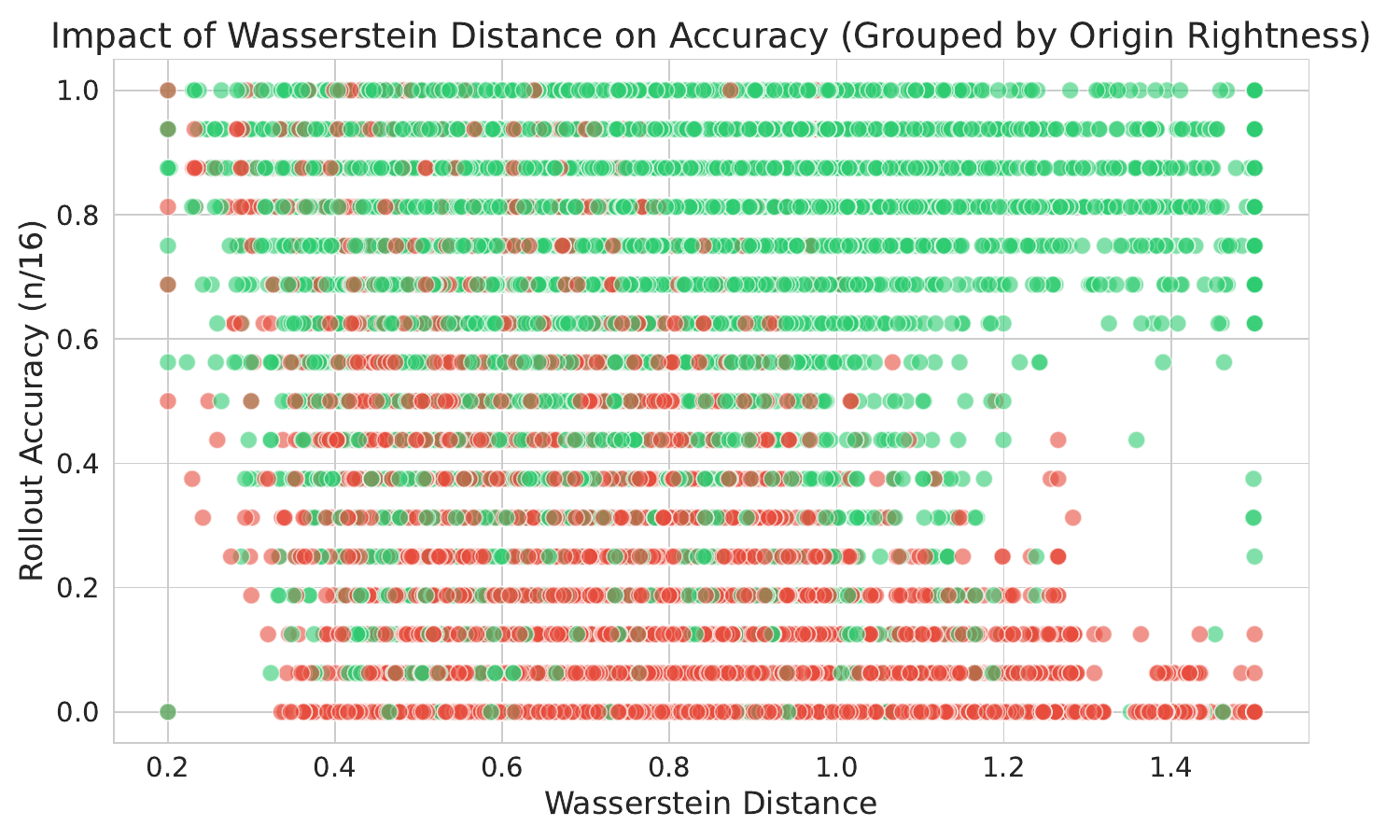}
    \end{subfigure}
    \hfill
    \begin{subfigure}{0.48\linewidth}
        \centering
        \includegraphics[width=1\linewidth]{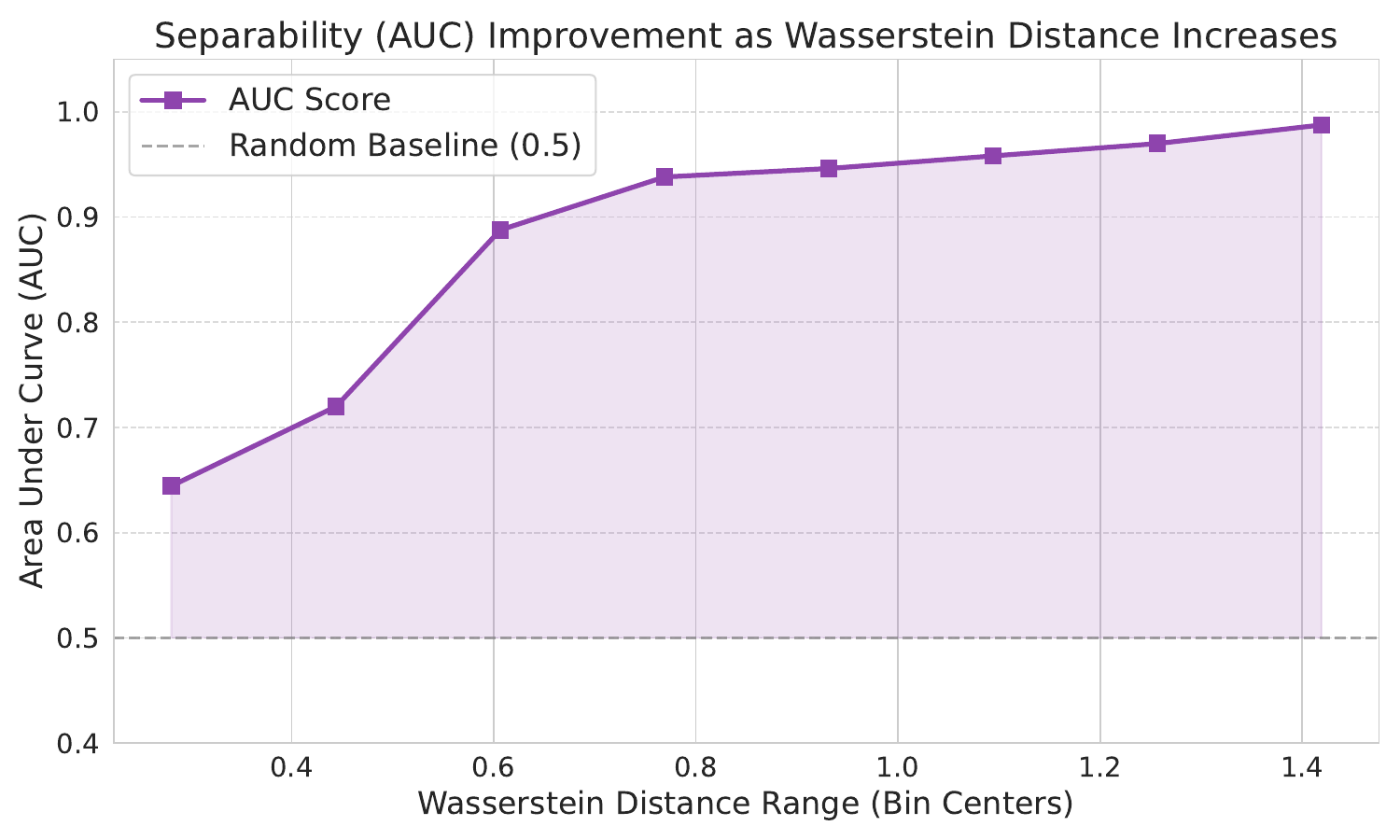}
    \end{subfigure}
    \hfill
    \begin{subfigure}{1\linewidth}
        \centering
        \includegraphics[width=1\linewidth]{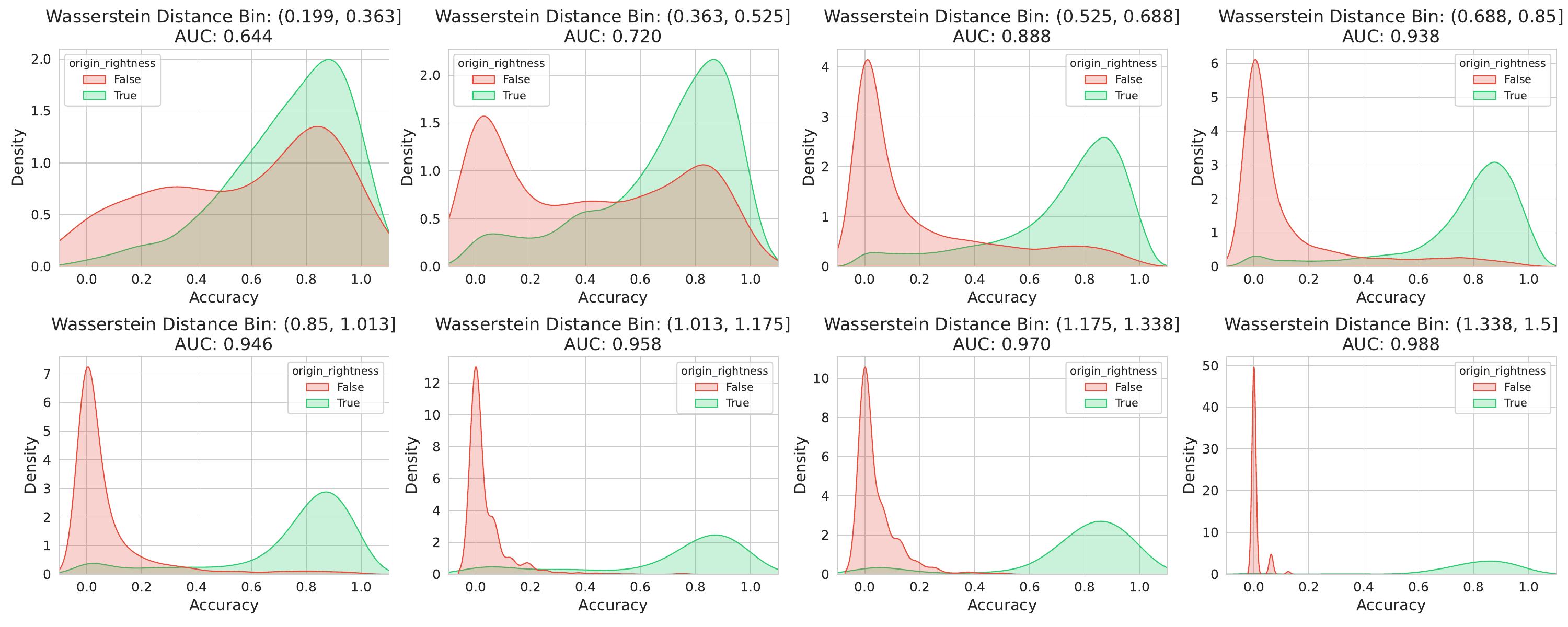}
    \end{subfigure}
    
    \caption{Empirical verification of the separation conditions. (Top Left) Empirical accuracy stratified by Wasserstein distance, grouped by original rollout correctness. (Top Right) Monotonic improvement in separability (AUC) between correct and incorrect spans as the Wasserstein distance magnitude increases. (Bottom) Per-bin density plots of accuracy distributions, visualizing the transition from highly overlapping to well-separated as the Wasserstein distance increases.}
    \label{fig:Empirical_Validation}
\end{figure}

Figure~\ref{fig:Empirical_Validation} (top left) plots each span as a dot, showing its Wasserstein distance to the opposing group versus its empirical continuation accuracy ($n/16$). Green and red dots denote spans from originally correct and incorrect rollouts, respectively. Within every Wasserstein distance bin, spans from correct rollouts achieve consistently higher empirical accuracy than those from incorrect rollouts, confirming that ground-truth correctness information persists through truncation and resampling. More importantly, higher Wasserstein distances are systematically associated with lower accuracy in both groups, extending the anti-correlation of Section~\ref{sec:observation} to a setting conditioned on the original outcome.

The critical question for our theory is whether this distributional signal is strong enough to reliably separate the two classes. Figure~\ref{fig:Empirical_Validation} (right) answers affirmatively: the AUC between correct and incorrect spans increases monotonically with the Wasserstein distance magnitude, rising from approximately 0.65 at low Wasserstein distances to above 0.97 at high Wasserstein distances. The per-bin density plots (Figure~\ref{fig:Empirical_Validation}, bottom row) visualize this transition—at low Wasserstein distances the accuracy distributions of the two groups overlap almost entirely, while at high Wasserstein distances they become well-separated. This monotonic growth in separability directly confirms that the population-level distributional difference $D(S)$ exceeds the finite-sample noise $\eta$ by a widening margin as the Wasserstein distance increases, providing concrete empirical support for the condition $D(S) > 2\eta(n,d)$ required by Theorem ~\ref{thm:separation}.

These results show that the theoretical separation is not merely a limiting guarantee but manifests at finite sample sizes with practical discriminative power. Since this discriminative power is derived entirely from the model's own hidden states, it validates the use of Wasserstein distance as a principled token-level weight in the GRPO objective without step-level annotations or an external reward model.
\subsection{Experimental Setup}
\paragraph{Backbones.} 
We evaluate on four backbones from two model families: Qwen2.5-Math-7B, Qwen2.5-Coder-7B, and Qwen2.5-14B-Base~\cite{yang2024qwen2.5}, and Llama3.1-8B-Instruct~\cite{Grattafiori2024llama}. For mathematical reasoning, we train Qwen2.5-Math-7B, Qwen2.5-14B-Base, and Llama3.1-8B-Instruct; for code generation, we train Qwen2.5-Coder-7B, Qwen2.5-14B-Base, and Llama3.1-8B-Instruct. The two task settings are trained independently.
\paragraph{Training datasets and settings.} 
For mathematical reasoning, we perform RL training on the MATH dataset~\cite{hendrycks2021Math} following the practice of~\cite{liu2025dr_grpo}. MATH contains 8.5K problems spanning algebra, geometry, counting, probability, number theory, and related areas. 
For code generation, we perform RL training on Eurus-RL-Code~\cite{cui2025prime} following~\cite{yang2026g_opd}. Eurus-RL-Code contains 25K samples. Detailed training settings are provided in Appendix~\ref{sec:setting_details}.

\paragraph{Evaluation.}
For mathematical reasoning, we evaluate on five competition-level benchmarks: AIME24~\cite{ai_mo_aime2024}, AIME25~\cite{opencompass_aime2025}, AMC23~\cite{math_ai_amc23}, MATH500~\cite{lightman2023math500}, and OlympiadBench~\cite{he2024OlympiadBench}. These benchmarks span a broad range of difficulty, from high-school competition problems (AMC23, MATH500) to olympiad-level problems (AIME24/25, OlympiadBench), and all require multi-step reasoning with a verifiable final answer. 
We report \texttt{avg@32} on AIME24 and AIME25, \texttt{avg@8} on AMC23, and \texttt{avg@1} on MATH500 and OlympiadBench. Here \texttt{avg@k}~\cite{wang202580/20,yang2025DCPO} denotes the average correctness over \(k\) independently sampled completions per problem under temperature sampling (\(T=0.1\), top-\(p=1.0\)). 

For code generation, we evaluate on five benchmarks: HumanEval, HumanEval+, MBPP, MBPP+~\cite{liu2023Humaneval_mbpp}, and LiveCodeBench (v6, February 2025 to May 2025)~\cite{jain2024livecodebench}. 
HumanEval+ and MBPP+ extend the original test suites with substantially more unit tests, thereby providing a stricter functional correctness signal and reducing false positives caused by under-specified tests. 
LiveCodeBench is included to probe generalization on contamination-free problems released after the base model's training cutoff. 
We report \texttt{avg@4} on all code benchmarks, where \texttt{avg@4} is defined analogously as the average correctness over four independently sampled completions per problem under temperature sampling.

\paragraph{Checkpoint Selection.}
Aligning with existing studies~\cite{zhang2025Critic-grpo,zhao2026opsd}, we use the following evaluation protocols.
We assess each run every 5 training steps and apply the same benchmark-based checkpoint selection protocol to all compared methods. 
For mathematical reasoning, we run each experiment with three random seeds. For each seed, we record the highest benchmark-averaged score across the five math benchmarks over all evaluated checkpoints, and report the mean and standard deviation of this best-observed score across seeds. For code generation, we analogously report the highest benchmark-averaged score over the five code benchmarks. This protocol measures best-observed performance under a uniform model-selection rule shared by all methods. 

\paragraph{Baselines.}
We compare our proposed SHEAR against three categories of baselines.
\begin{itemize}[leftmargin=1.5em]
\item \textbf{Base RLVR algorithm.} GRPO~\citep{shao2024grpo}, the canonical multi-sample policy-gradient method and the underlying backbone of many reasoning RL pipelines.
    \item \textbf{Intrinsic-signal-guided credit assignment.} 
    \textit{Entropy Advantage Reshaping}~\citep{cheng2025Entropy_adv}, a non-PRM approach that reshapes token-level advantages using entropy signals derived from the policy model itself.
    \item \textbf{PRM-based credit assignment.} 
    \textit{PURE}~\citep{cheng2025pure}, which leverages process reward models to improve reasoning and mitigate reward hacking, and \textit{PRM-based Advantage Reshaping}, which enhances GRPO with backward-filled PRM signals for finer-grained token-level credit assignment. 
    PRM-based baselines are evaluated only on mathematical reasoning tasks, for which we use Qwen2.5-Math-PRM-7B~\cite{zhang2025qwen2.5_math_prm} as the process reward model.
\end{itemize}

\subsection{Main Results}
\paragraph{Overall performance.}
Table~\ref{tab:overall} and Table~\ref{tab:code} report results on five mathematical reasoning benchmarks and five code generation benchmarks across three backbones.
On mathematical reasoning, SHEAR achieves the strongest average performance among the compared methods on all three backbones, consistently improving over standard GRPO. The gains are especially visible on AIME24 and AMC23, two benchmarks that require sustained multi-step reasoning. On code generation, SHEAR also improves over GRPO on all three backbones, although the margins are smaller than those observed on mathematical reasoning.
\paragraph{Comparison with PRM-based methods.} 
A central question is whether a self-supervised hidden-state signal can compete with explicit step-level supervision from a trained process reward model. 
Empirically, the answer is yes on the mathematical reasoning benchmarks considered here. As shown in Table~\ref{tab:overall}, SHEAR outperforms both PRM(Reshape adv.) and PRM(PURE) on all three backbones in terms of average score. The contrast is particularly clear on Qwen2.5-14B-Base, where both PRM variants fall below vanilla GRPO by a substantial margin. 
We attribute this to a distribution mismatch between the externally PRM and the evolving policy: as the policy drifts during training, the PRM's step-level scores become increasingly miscalibrated, injecting noise that can outweigh the benefit of denser supervision. This effect is amplified on the 14B backbone which is reflected in the larger standard deviations of PRM variants. By deriving the credit signal directly from the policy model's own hidden-state representations, SHEAR avoids dependence on a separately trained reward model.
\paragraph{Comparison with Entropy adv.} 
Entropy adv. is the most relevant baseline conceptually, as it also performs token-level advantage reweighting without external supervision. SHEAR outperforms it on every backbone for mathematical reasoning. This supports our hypothesis that distributional divergence in hidden states carries strictly richer process-level information than per-token entropy, which conflates uncertainty with reasoning quality and treats high-entropy exploratory tokens identically to high-entropy erroneous ones.
\paragraph{Code generation.} As shown in Table~\ref{tab:code}, SHEAR yields modest but consistent improvements over GRPO across all three backbones, with the most visible gains appearing on LiveCodeBench, the longest and most complex test set in the suite. The smaller margins on HumanEval and MBPP are consistent with the relatively short solution length of these benchmarks, which limits the resolution at which span-level credit assignment can act. These results suggest that SHEAR's benefits are most pronounced when reasoning chains are long enough for localized errors to be distinguishable from globally erroneous trajectories—precisely the regime motivating our method.

\paragraph{Training dynamics.} 
Figure~\ref{fig:avg_score_training} plots the benchmark-averaged score over the five mathematical reasoning benchmarks throughout training for the three backbones. Solid lines denote the mean across three random seeds, and shaded regions denote $\pm 1$ standard deviation. 
While SHEAR and the GRPO show similar in early progress, SHEAR consistently pulls ahead after the initial training phase (typically around 100-150 steps) and maintains a higher overall accuracy towards the end, with no signs of instability or collapse. This sustained and robust improvement under SHEAR is consistent with the intuition that concentrating gradient updates on tokens whose hidden states are distributionally distinctive reduces the noise introduced by uniformly weighted credit.

\begin{table*}[t]
\centering
\setlength{\tabcolsep}{4pt}
\small
\caption{
Overall performance on mathematical reasoning benchmarks. We report avg@32 on AIME24 and AIME25, avg@8 on AMC23, and avg@1 on MATH500 and OlympiadBench. Entries report mean ± std over three random seeds. 
}
\label{tab:overall}
\begin{tabular}{lccccccc}
\toprule
Method & AIME24 & AIME25 & AMC23 & MATH500 & Olympiad & Avg \\
\midrule
\multicolumn{7}{c}{\textbf{Qwen2.5-Math-7B}}\\
\midrule
GRPO   & 33.6 \std{1.9} & 15.5 \std{0.6} & 66.5 \std{1.2} & 84.9 \std{1.7} & 47.8 \std{0.1} & 49.6 \std{0.8} \\
Entropy adv. & 32.8\std{1.8} & 16.7\std{1.8} & 66.0\std{0.6} & 83.4\std{0.3} & 48.6\std{0.8} & 49.5\std{0.5}   \\
PRM(Reshape adv.) & 34.1\std{3.2} & 15.2\std{1.5} & 66.8\std{1.7} & 83.1\std{1.7} & 48.1\std{1.8} & 49.6\std{1.3}   \\
PRM(PURE)& 24.5\std{0.6} & 16.2\std{2.7} & 70.0\std{2.5} & 82.5\std{1.1} & 47.6\std{1.2} & 48.2\std{0.6}   \\
SHEAR & 35.2 \std{2.3} & 16.8 \std{0.5} & 70.3 \std{3.2} & 83.3 \std{0.8} & 48.7 \std{1.0} & 51.2 \std{0.8} \\
\midrule
\multicolumn{7}{c}{\textbf{Llama3.1-8B-Instruct}}\\
\midrule
GRPO   & 7.9 \std{0.7}  & 0.7 \std{0.2}  & 32.7 \std{1.2} & 57.7 \std{1.5} & 24.5 \std{1.3} & 24.7 \std{0.4} \\
Entropy adv. & 11.7\std{2.2} & 1.5\std{1.3} & 33.1\std{2.1} & 57.0\std{2.5} & 22.1\std{0.1} & 25.1\std{0.0}   \\
PRM(Reshape adv.) & 12.3\std{2} & 1.1\std{0.4} & 33.4\std{1.5} & 55.2\std{0.5} & 23.4\std{1.5} & 25.1\std{0.9}   \\
PRM(PURE)& 8.1\std{5.1} & 1.2\std{1.4} & 27.9\std{3.4} & 54.6\std{1.7} & 20.5\std{1.1} & 22.5\std{0.3}   \\
SHEAR & 12.7 \std{1.8} & 2.1 \std{1.4}  & 37.2 \std{0.2} & 56.3 \std{1.3} & 22.8 \std{0.1} & 26.2 \std{0.1} \\
\midrule
\multicolumn{7}{c}{\textbf{Qwen2.5-14B-Base}}\\
\midrule
GRPO   & 16.5 \std{1.8} & 15.0 \std{1.6} & 62.4 \std{1.3} & 83.3 \std{1.8} & 49.9 \std{0.5} & 45.4 \std{0.4} \\
Entropy adv. & 18.5\std{1.5} & 16.4\std{1.7} & 64.8\std{1.8} & 82.2\std{0.4} & 49.4\std{1.5} & 46.3\std{0.1}   \\
PRM(Reshape adv.) & 12.8\std{3.1} & 10.6\std{2.8} & 58.5\std{3.2} & 79.7\std{2.4} & 44.1\std{2.7} & 41.2\std{2.3}   \\
PRM(PURE)& 13.0\std{1.6} & 8.0\std{1.8} & 49.2\std{1.2} & 76.1\std{0.1} & 40.7\std{0.1} & 37.4\std{0.1}   \\
SHEAR & 18.7 \std{1.5} & 16.9 \std{1.1} & 62.2 \std{3.9} & 84.5 \std{1.1} & 52.0 \std{1.1} & 46.9 \std{1.0} \\
\bottomrule
\end{tabular}
\end{table*}
\begin{figure}[ht]
    \centering
    \begin{subfigure}{0.32\linewidth}
        \centering
        \includegraphics[width=1\linewidth]{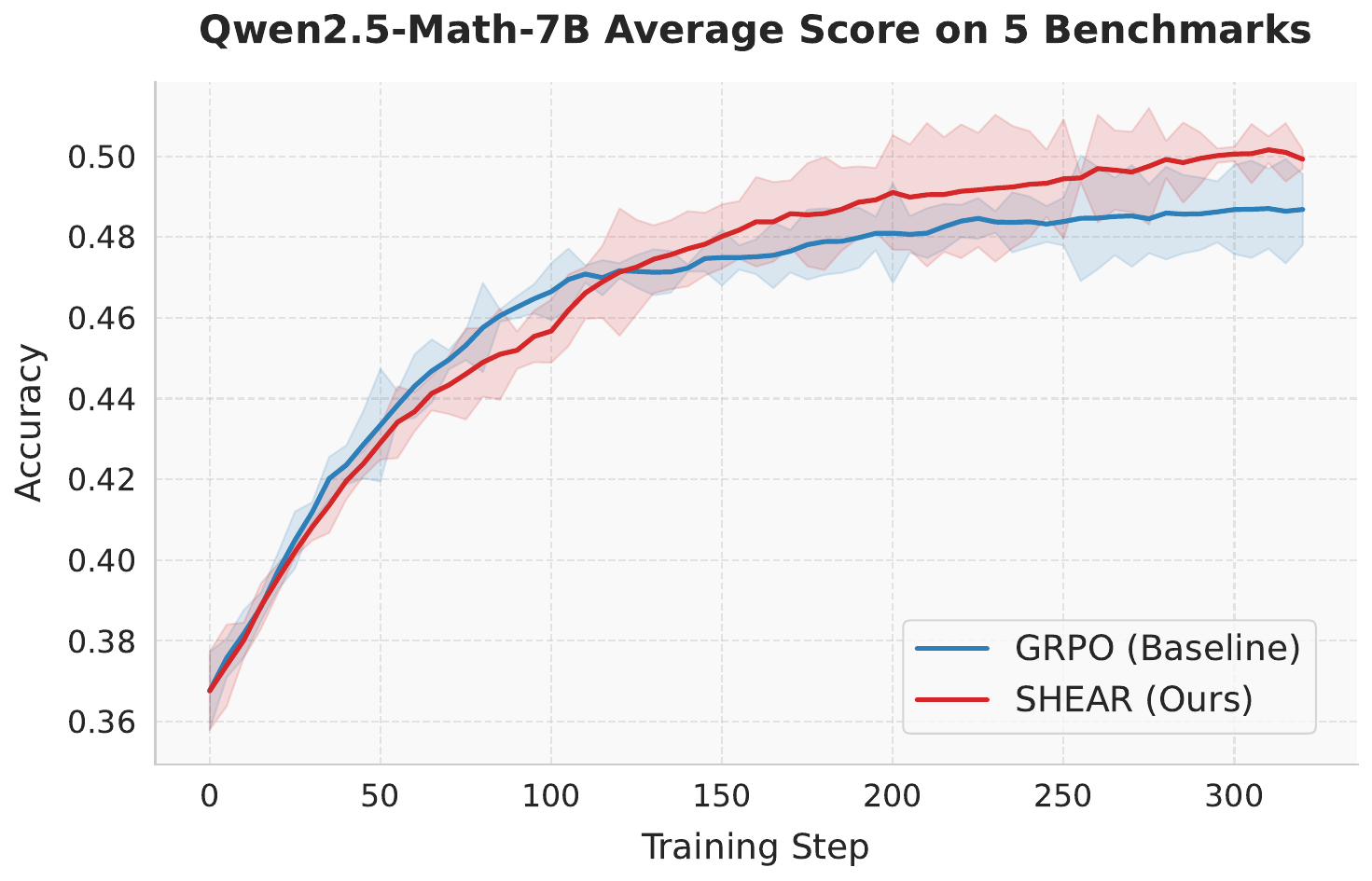}
    \end{subfigure}
    \hfill
    \begin{subfigure}{0.32\linewidth}
        \centering
        \includegraphics[width=1\linewidth]{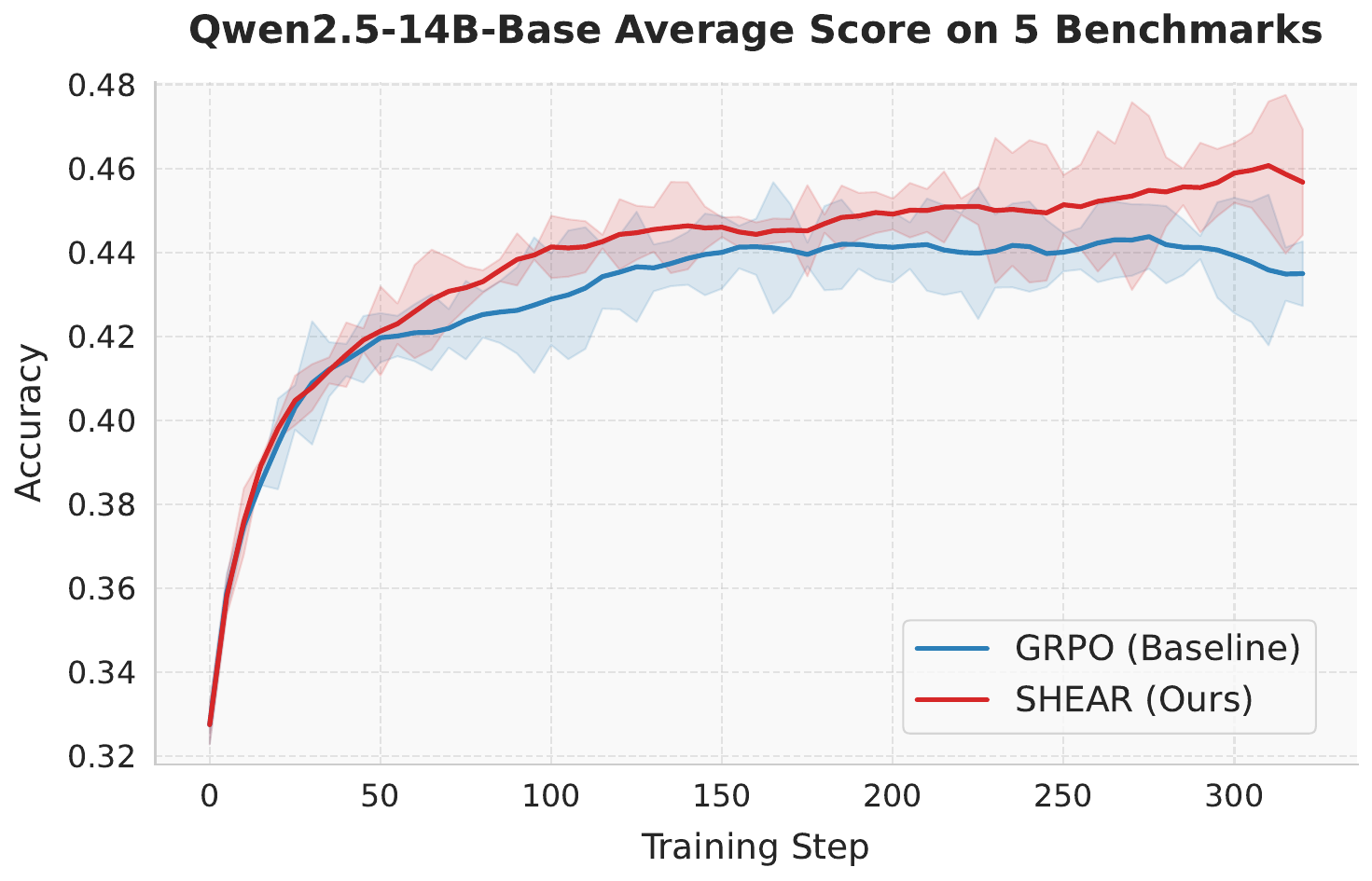}
    \end{subfigure}
    \hfill
    \begin{subfigure}{0.32\linewidth}
        \centering
        \includegraphics[width=1\linewidth]{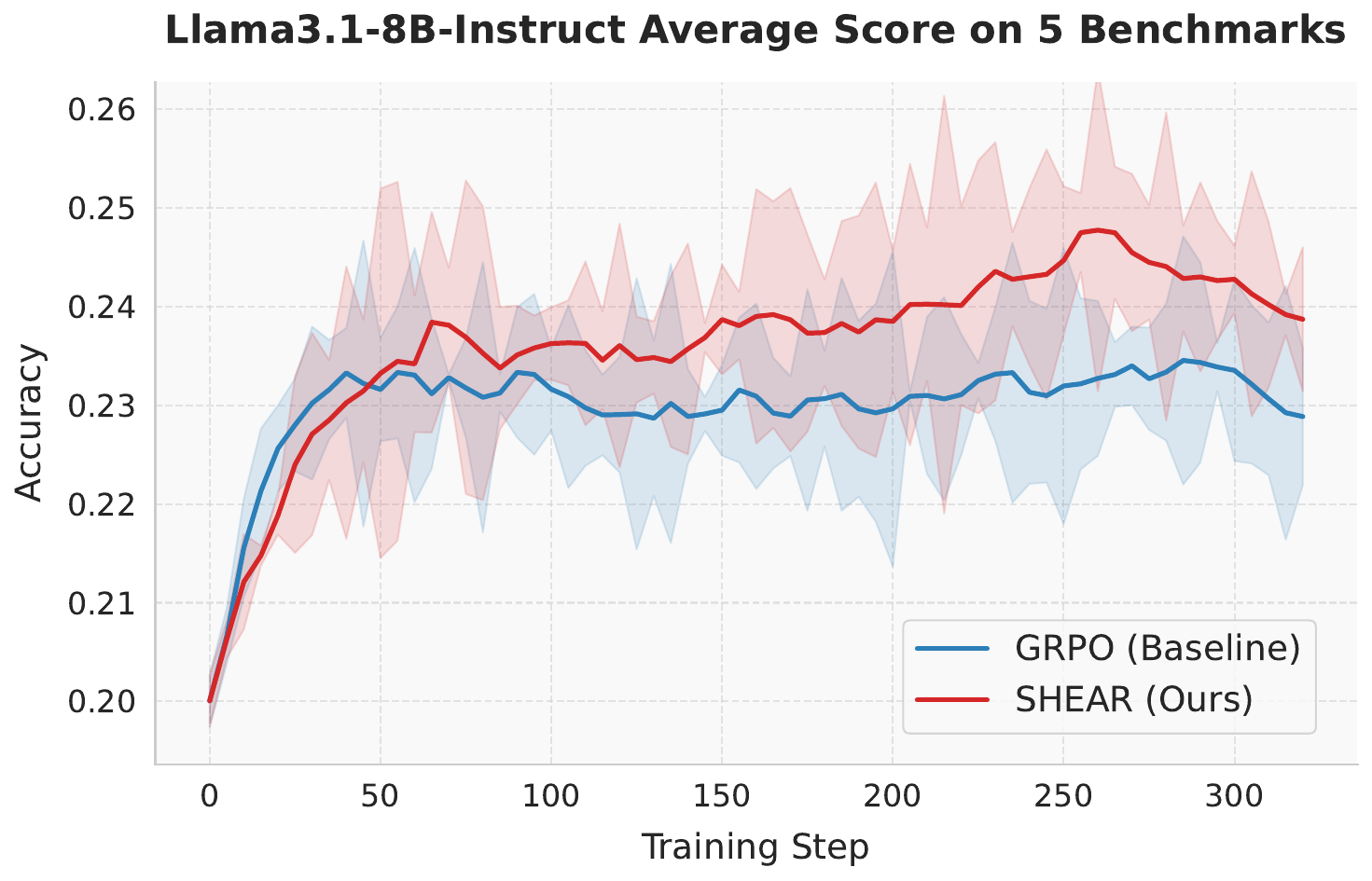}
    \end{subfigure}

    \caption{
    Training dynamics on mathematical reasoning. Each panel shows the benchmark-averaged score over the five math benchmarks as a function of training steps for one backbone. Solid lines denote the mean across three random seeds, and shaded regions denote ±1 standard deviation.
    }
    \label{fig:avg_score_training}
\end{figure}

\begin{table*}[t]
\centering
\setlength{\tabcolsep}{4pt}
\small
\caption{
Overall performance on code generation benchmarks. We report avg@4 on all benchmarks.
}
\label{tab:code}
\begin{tabular}{lccccccc}
\toprule
Method & HumanEval & HumanEval+ & MBPP & MBPP+ & LiveCodeBench & Avg \\
\midrule
\multicolumn{7}{c}{\textbf{Qwen2.5-Coder-7B}}\\
\midrule
GRPO   &75.8 &71.0 &81.8 &69.3 &17.8 &63.1  \\
SHEAR &76.1 &72.1 &82.2 &70.8 &18.3 &63.9  \\
\midrule
\multicolumn{7}{c}{\textbf{Llama3.1-8B-Instruct}}\\
\midrule
GRPO   &65.7 &60 &70.7 &59.7 &17.4 &54.7  \\
SHEAR &70.0 &63.5 &69.5 &58.4 &17.7 &55.8  \\
\midrule
\multicolumn{7}{c}{\textbf{Qwen2.5-14B-Base}}\\
\midrule
GRPO   &85.0 &79.3 &82.8 &68.8 &19.7 &67.1  \\
SHEAR &85.5 &79.6 &83.1 &69.0 &21.1 &67.7  \\
\bottomrule
\end{tabular}
\end{table*}

\subsection{Ablation Studies}
We conduct ablations to isolate the contribution of two key design choices: (i) the choice of distance metric used to measure span-level distributional divergence, and (ii) the role of global norm normalization in Algorithm~\ref{alg:method}. 
\label{sec:exp-ablation}
\subsubsection{Distance Measure Comparison}
\label{sec:exp:distance}
\paragraph{Setup.} Section~\ref{sec:theory} motivates the use of the 1-Wasserstein distance on the grounds that, via Kantorovich–Rubinstein duality, it captures all 1-Lipschitz functionals of the distributional difference, including mean shift, variance, and higher-order shape changes. We test this claim empirically by replacing the Sinkhorn-approximated $W_1$ in Algorithm~\ref{alg:method} with three alternatives that differ in how much distributional information they retain:
\begin{enumerate}[label=(\roman*)]
    \item Cosine distance between span means: a pure mean-shift, direction-only signal that discards both magnitude and shape.
    \item Chamfer distance between span hidden-state sets: a point-cloud distance that captures geometric coverage but is not a proper transport metric and is sensitive to outliers.
    \item MMD (Maximum Mean Discrepancy) with an RBF kernel: a kernel embedding distance that captures all moments through the kernel feature map but lacks the metric structure of optimal transport.
\end{enumerate}


All variants use identical span construction, normalization, max-pooling, and weighting; only the inner distance computation differs.
\paragraph{Result.} Fig.~\ref{fig:distance_ablation} shows the training curves for the four variants alongside the GRPO baseline. Wasserstein, Chamfer, and MMD form a top tier with peak accuracies above the GRPO baseline.
SHEAR (Wasserstein) achieves the highest peak, but the margin over SHEAR (Chamfer) and SHEAR (MMD) is small, suggesting that any sufficiently expressive distributional metric is able to capture the dominant component of the credit signal. This is consistent with the theoretical view in Section~\ref{sec:theory}: the population-level signal $D(S)$ reflects a distributional gap that is rich enough to be detected by multiple measures, with optimal-transport providing the tightest characterization but not the only viable one.

The behavior of SHEAR (Cosine) is more revealing. 
SHEAR (Cosine) plateaus at $0.4837$, falling below the GRPO baseline ($0.4884$) throughout the latter half of training.
This is a striking negative result: using only the angular separation between span means as a credit signal actively hurts learning. We interpret this as evidence that mean-direction shifts in hidden states are not a faithful proxy for reasoning quality; they correlate with surface-level lexical variation and other reasoning-irrelevant factors, and weighting tokens by such a signal injects noise into the policy gradient. The fact that the three distribution-aware metrics succeed where the mean-only metric fails directly supports the central claim of Section~\ref{sec:theory}: capturing the full distributional difference, rather than a low-order summary, is what makes hidden states a reliable source of process-level credit.

\begin{figure}[t]
\centering
\begin{subfigure}[t]{0.48\textwidth}
    \includegraphics[width=\linewidth]{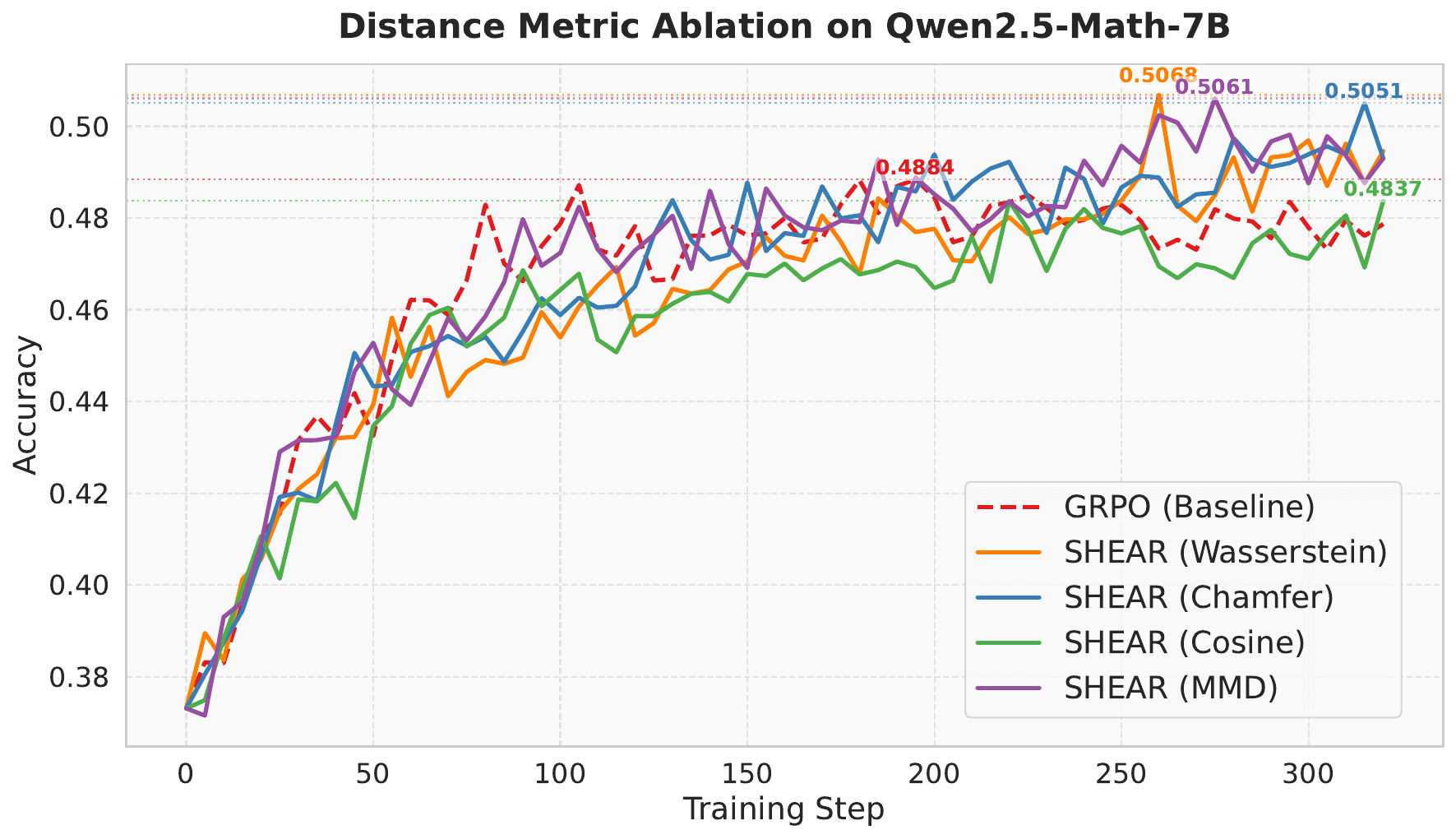}
    \caption{Distance metric ablation.}
    \label{fig:distance_ablation}
\end{subfigure}
\hfill
\begin{subfigure}[t]{0.48\textwidth}
    \includegraphics[width=\linewidth]{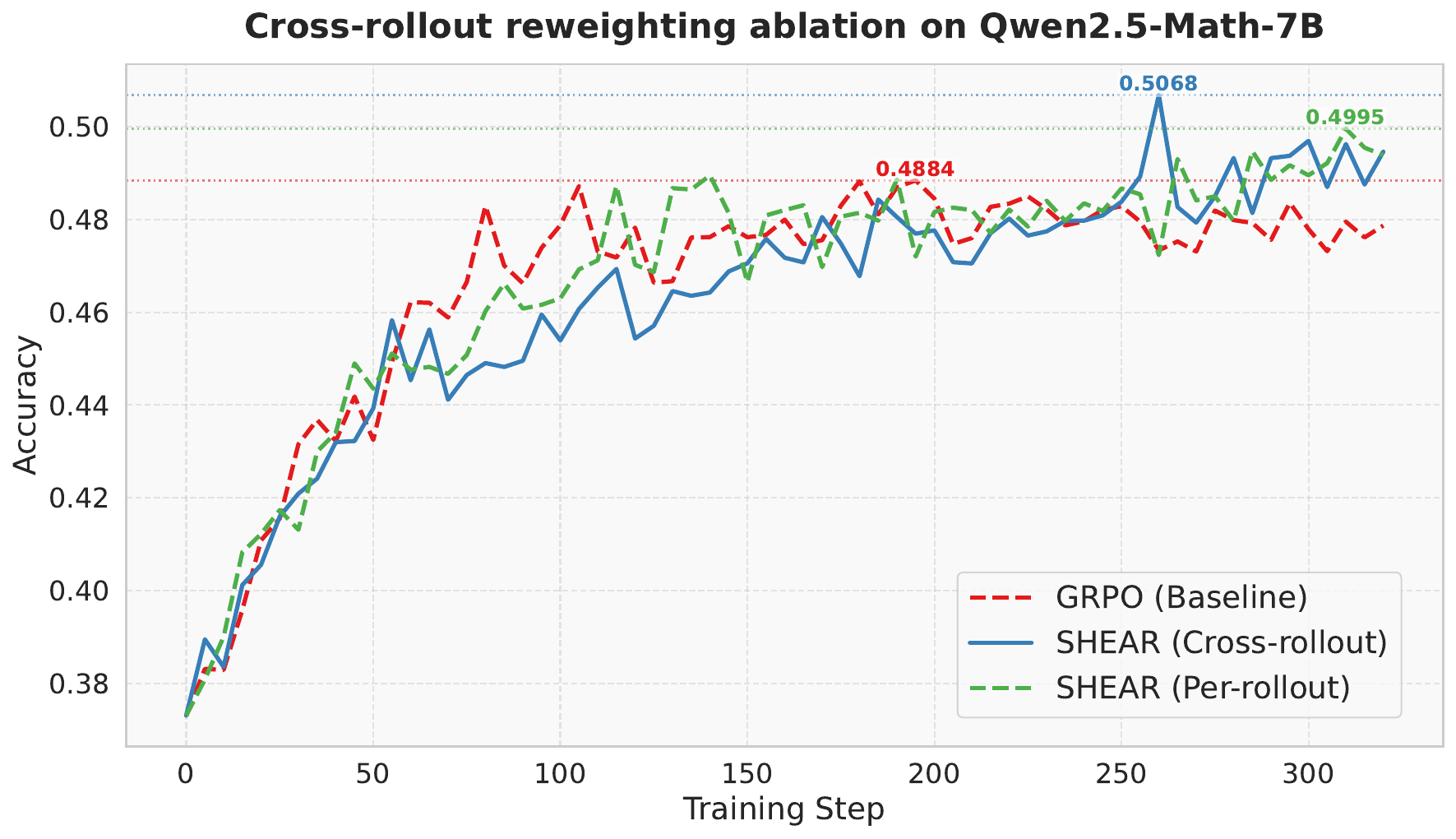}
    \caption{Cross-rollout reweighting ablation.}
    \label{fig:norm_ablation}
\end{subfigure}
\caption{Ablation studies on Qwen2.5-Math-7B.}
\label{fig:ablations}
\end{figure}

\subsubsection{Cross-Rollout Reweighting Ablation}
\paragraph{Setup.} Algorithm~\ref{alg:method} normalizes span distances by the global mean hidden-state norm $\bar{n}$ computed across the entire group. This rescaling is intended to stabilize magnitudes across varying latent scales, but it has a side effect: because $\bar{n}$ is shared across rollouts rather than computed per rollout, the total gradient magnitude received by each rollout still depends on the absolute scale of its span distances. This introduces a form of cross-rollout reweighting that is conceptually distinct from the within-rollout token reweighting that the method is designed to deliver. A natural question is whether this cross-rollout effect contributes meaningfully to the observed gains, or whether the method would behave differently if span distances were rescaled in a way that suppresses it.
To probe this, we compare our default formulation against a variant in which the cross normalization step is replaced by per rollout normalization:
\begin{enumerate}[label=(\roman*)]
    \item SHEAR (cross-rollout) (\emph{default}): span distances are divided by the global mean hidden-state norm $\bar{n}$ as in Algorithm~\ref{alg:method}.
    \item SHEAR (per-rollout): each span distance is rescaled by a per-rollout normalizer, equalizing the average weight magnitude across rollouts within a group. This explicitly removes any cross-rollout reweighting and isolates the within-rollout token-level signal.
\end{enumerate}
\paragraph{Result.}As shown in Fig~\ref{fig:norm_ablation}, both variants of SHEAR substantially outperform the GRPO baseline. Two observations follow.
First, the dominant source of improvement is clearly the within-rollout token-level credit signal: even after explicitly equalizing per-rollout weight magnitudes, SHEAR (per-rollout) still improves over GRPO by a comfortable margin. This indicates that the method's effectiveness does not hinge on cross-rollout magnitude differences—the per-token ranking of credit within a rollout, which the separation theorem in Section~\ref{sec:theory} directly addresses, is what carries the learning signal. 
Second, the default formulation nonetheless retains a measurable edge over the per-rollout-equalized variant, suggesting that allowing rollouts with larger absolute distributional gaps to receive proportionally larger total gradient mass is mildly beneficial. Such rollouts contain more discriminable reasoning errors and plausibly merit more aggressive updates; collapsing the magnitude variation across rollouts gives up a small amount of useful signal. We therefore adopt SHEAR (cross-rollout) as the default in all main experiments. Importantly, both variants sit comfortably above the GRPO baseline, indicating that the method is robust to this design choice.
This ablation reinforces the interpretation in Section~\ref{sec:theory}: the value of the Wasserstein signal lies primarily in where within a rollout it concentrates credit, with cross-rollout magnitude variation playing a complementary role.

\subsection{Effect of Span and Stride Sizes}
\begin{figure}[ht]
    \centering
    \includegraphics[width=0.5\linewidth]{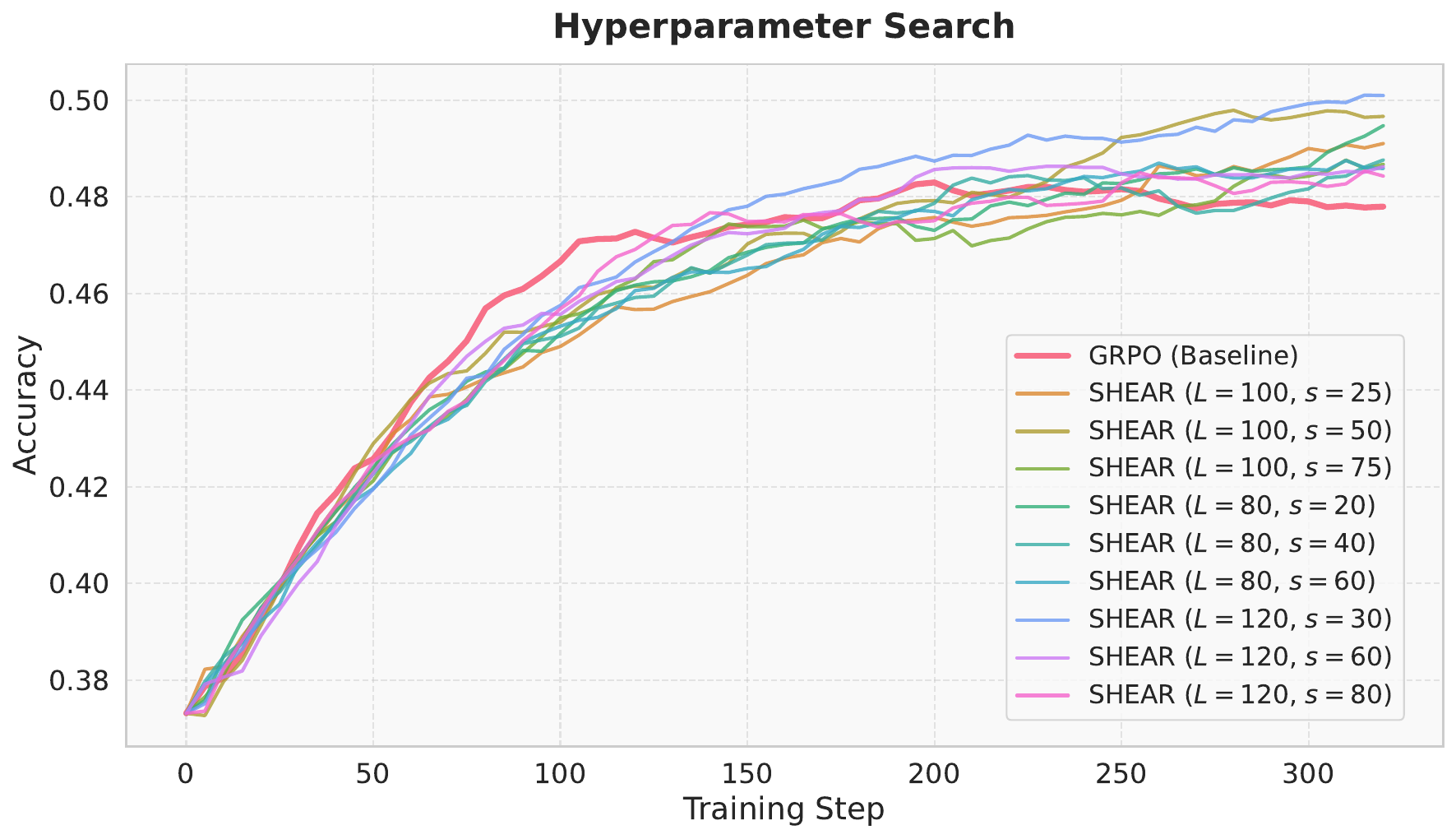}
    \caption{
    Average accuracy across five math benchmarks for varying span length $w \in \{80, 100, 120\}$ and stride $s$. Performance is robust to span length but sensitive to stride where smaller strides yield consistently higher accuracy, suggesting that finer-grained span overlap improves credit assignment resolution.
    }
    \label{fig:hyperparameter_search}
\end{figure}
Figure~\ref{fig:hyperparameter_search} reports the effect of span length $w \in \{80, 100, 120\}$ and stride $s$ on average accuracy across the five math benchmarks, using Qwen2.5-Math-7B. Two patterns emerge.
First, performance is relatively insensitive to span length: for a fixed stride, varying $w$ from 80 to 120 produces only marginal differences in final accuracy, indicating that the Wasserstein signal is robust across a range of local window sizes.
Second, performance is more sensitive to stride size. Within each span length, smaller strides consistently yield higher and more stable accuracy curves. For example, at $w = 100$, the configuration with $s = 25$ outperforms $s = 75$ throughout training. This is expected: a smaller stride increases overlap between consecutive spans, providing finer-grained coverage of the reasoning chain and reducing the risk that a localized divergence event falls between non-overlapping windows. We adopt $w = 100$, $s = 25$ as the default configuration in all main experiments.

\subsection{Effect of Rollout Sample Size}
\begin{figure}[ht]
    \centering
    \includegraphics[width=0.5\linewidth]{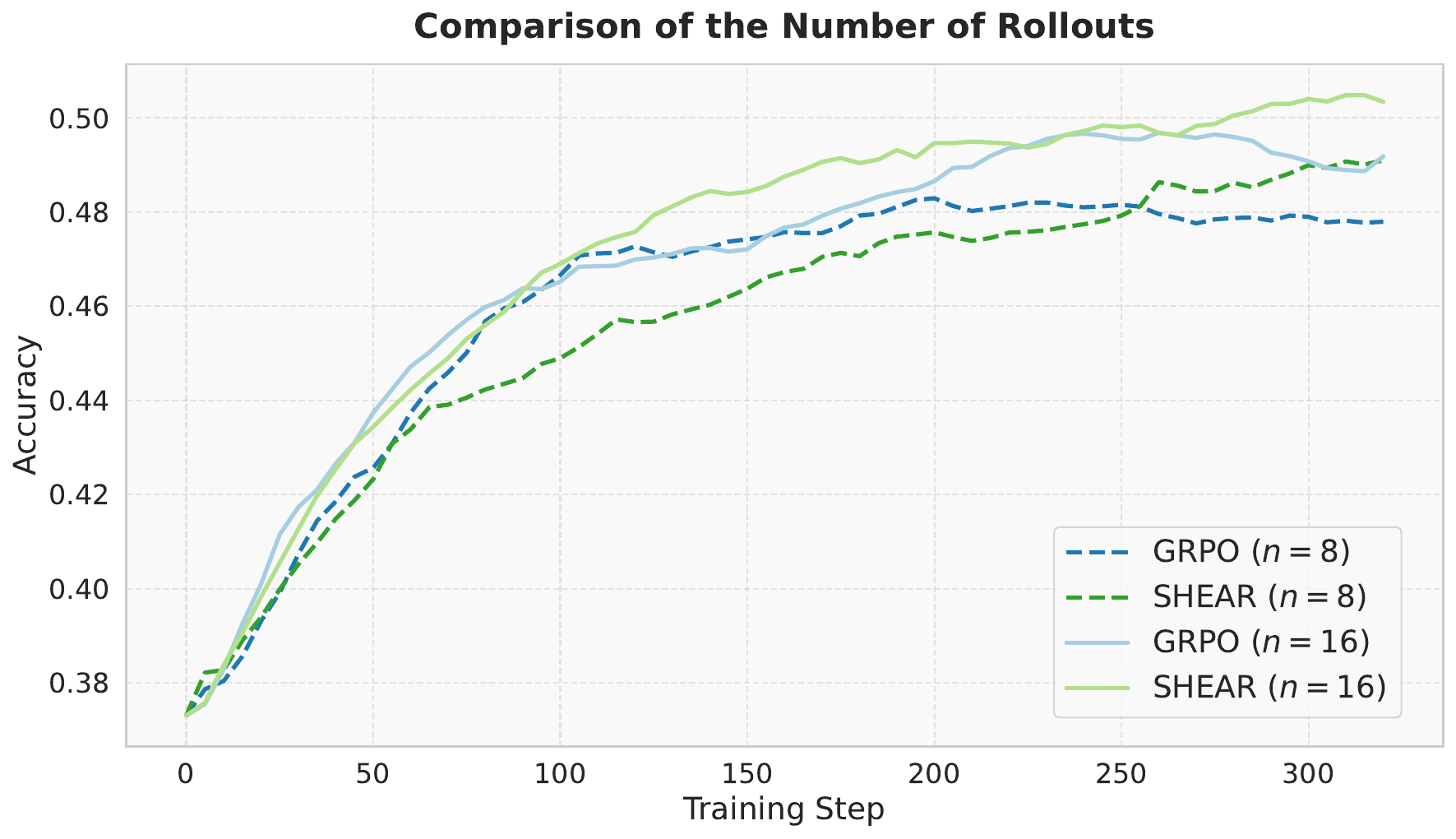}
    \caption{Impact of rollout group size ($G = 8$ vs.\ $G = 16$) on Qwen2.5-Math-7B. Both methods benefit from larger groups, but SHEAR exhibits a wider improvement margin, suggesting that richer opposing sets enhance the span-level Wasserstein signal.}
    \label{fig:rollout_compariosn}
\end{figure}
As illustrated in Fig~\ref{fig:rollout_compariosn}, we compares SHEAR and GRPO under rollout group sizes $G = 8$ and $G = 16$ on Qwen2.5-Math-7B. Two observations stand out. First, increasing the group size from 8 to 16 improves both methods, confirming that larger rollout groups provide richer contrastive signal for credit assignment. Second, the gap between SHEAR and GRPO widens at $G = 16$: SHEAR ($G = 16$) achieves the highest final accuracy among all four configurations, while GRPO ($G = 16$) plateaus below SHEAR ($G = 8$) in late training. This suggests that SHEAR benefits more from additional rollouts than standard GRPO, likely because a larger opposing set yields more reliable span-level Wasserstein estimates, amplifying the advantage of fine-grained credit assignment.
\section{Computational Overhead}\label{app:comp_overhead}
\begin{figure}[ht]
    \centering
    \includegraphics[width=0.5\linewidth]{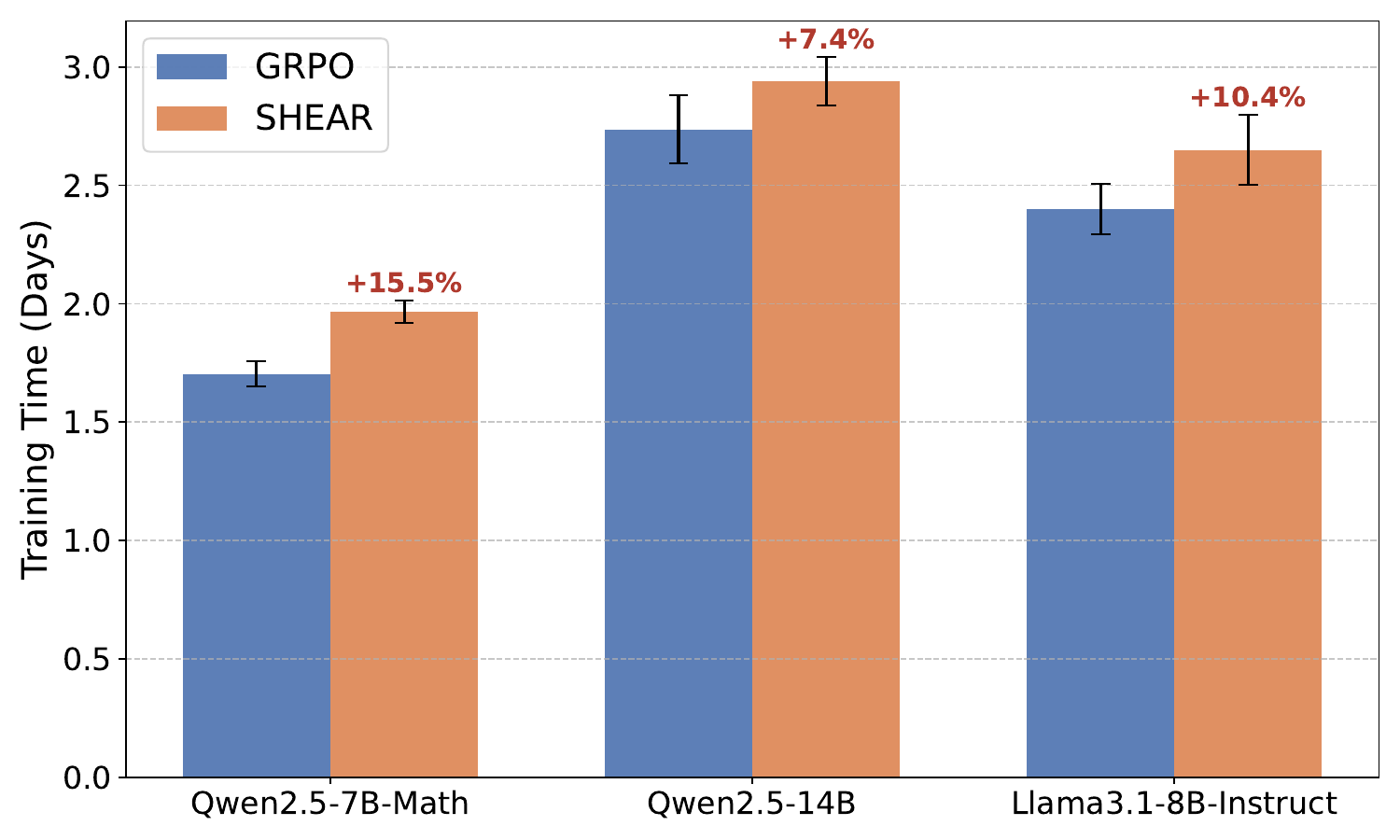}
    \caption{Training time overhead of SHEAR relative to standard GRPO. Percentages indicate the relative increase in total training time. Error bars denote standard deviation across three random seeds. The overhead decreases with model scale and remains below 16\% in all cases.}
    \label{fig:comp_overhead}
\end{figure}
Computing span-level Wasserstein distances via the Sinkhorn algorithm introduces additional computation at each training step. Figure~\ref{fig:comp_overhead} quantifies this cost: relative to the GRPO baseline, SHEAR increases total training time by 15.5\% on Qwen2.5-Math-7B, 10.4\% on Llama3.1-8B-Instruct, and 7.4\% on Qwen2.5-14B-Base. The overhead decreases at larger model scale because the Sinkhorn computation operates on fixed-size spans regardless of model dimension, while the dominant cost of forward and backward passes grows with model size. In all cases, the overhead remains below 16\%, making SHEAR practical for standard RLVR training pipelines.

\section{Related Work}
\paragraph{Credit assignment in RLVR.} Credit assignment, namely identifying which tokens or intermediate steps in a reasoning chain contribute to the final outcome, is a central challenge in reinforcement learning with verifiable rewards. GRPO~\cite{shao2024grpo} assigns uniform advantages across all tokens which is coarse-grained~\cite{li2026OAR}. Process reward models (PRMs)~\cite{cheng2025pure}~\cite{cui2025prime}~\cite{lightman2023math500} provide finer-grained step-level supervision and have shown clear benefits over outcome-level feedback. However, PRMs require either human-annotated intermediate labels or an auxiliary reward model trained on such annotations, introducing both annotation cost and a potential transfer gap when the policy model's capacity exceeds that of the PRM. Cui et al.~\cite{cui2025prime} proposed learning process rewards through implicit signals, while Cheng et al.~\cite{cheng2025pure} introduced min-form credit assignment to mitigate reward hacking. Our work departs from the PRM paradigm entirely by extracting process-level credit from the policy model's own hidden states, requiring no additional annotation or reward model.
\paragraph{Token-level advantage reshaping.} Several recent methods seek to differentiate token-level advantages within GRPO relying on internal signals. ~\cite{cheng2025Entropy_adv}~\cite{zhang2025edge-grpo} exploit per-token entropy to reshape advantages, effectively improving credit granularity and exploration. Despite these benefits, a shared limitation remains: a scalar entropy metric cannot reliably distinguish between productive reasoning pivots and genuine erroneous steps. To overcome this representational bottleneck, recent methods leverage continuous latent representations to model and optimize reasoning~\cite{kang2026ladir}~\cite{du2026lto}~\cite{yue2025hrpo}. These works share a core insight: high-dimensional latent spaces inherently capture richer semantic signals that clearly separate correct reasoning patterns from erroneous ones. By operating in this continuous space—via latent diffusion ~\cite{kang2026ladir}, direct latent reward modeling ~\cite{du2026lto}, or hybrid policy optimization ~\cite{yue2025hrpo}, these approaches enable more precise credit assignment and diverse exploration than scalar-driven heuristics. Building upon this shared insight, SHEAR directly computes the span-level Wasserstein distance between the model's native hidden state, rather than designing complex latent generation architectures or training auxiliary latent reward models. This allows us to extract a highly expressive, self-supervised credit signal that shares the fine-grained evaluation advantages of continuous spaces, while eliminating the reliance on auxiliary models.


\section{Conclusion}
This paper introduces SHEAR, a self-supervised credit assignment method for RLVR that exploits the model's own hidden-state representations to provide fine-grained, token-level training signals. We first identify an empirical phenomenon that span-level Wasserstein distance between hidden-state distributions of correct and incorrect rollouts tracks local reasoning quality, without requiring step-level annotation or an external reward model. We formalize this observation with a separation theorem showing that post-divergence spans exhibit provably larger Wasserstein distances than pre-divergence spans, under a verifiable condition on the population-level distributional gap. Moreover, we translate this insight into a practical algorithm that rescales token-level GRPO advantages using span-level Wasserstein distances, requiring only minimal modifications to existing training pipelines. Experiments across five mathematical reasoning benchmarks and five code generation benchmarks demonstrate consistent improvements over standard GRPO and competitive or superior performance relative to PRM-based methods that rely on a separately trained 7B reward model, while incurring less than 16\% computational overhead.


\appendix
\section{Detailed Analysis of the Wasserstein Signal}
While Section~\ref{sec:observation} establishes a strong macro-level anti-correlation between empirical accuracy and Wasserstein distance, this appendix investigates the robustness and dynamic sensitivity of this signal. Specifically, we ask two questions: Is the W-distance signal merely an artifact of longer sequence lengths? And how does it respond to shifts in reasoning quality as the trajectory unfolds?

\subsection{Cross-Sectional Analysis: Controlling for Sequence Length}
As a reasoning chain progresses, empirical accuracy naturally decreases while hidden state drift increases. To ensure W-distance is an independent indicator of reasoning quality rather than a simple proxy for token depth, we conducted a cross-sectional analysis as shown in Fig~\ref{fig:observation_fix_position}.

By strictly grouping trajectories into distinct prefix position bins and dividing them into W-distance quantiles, a consistent pattern emerges. Notably, even in the early stages of reasoning, W-distance already demonstrates a preliminary discriminative ability: within the same early position group, a higher W-distance weakly but consistently corresponds to lower mean empirical accuracy. More importantly, trajectories that fall into high W-distance bins early in the process are significantly more prone to error accumulation, leading to more frequent and severe reasoning failures at later positions. This confirms that an elevated W-distance serves as an effective early warning signal for vulnerable reasoning state.

\subsection{Dynamic Step-wise Tracking of Reasoning Mutations}
To verify that W-distance captures structural shifts in reasoning, we analyzed the step-wise derivatives of both metrics ($\Delta$Accuracy and $\Delta$W-Distance) across consecutive strides. 
\paragraph{Computing $\Delta$Accuracy and $\Delta$W-Distance.} Using the same per-prefix measurements from Section~\ref{sec:observation}, we compute first-order differences between consecutive prefix positions within each rollout:
\[\Delta \text{Acc}_t = \text{Acc}(s_t) - \text{Acc}(s_{t-1}), \quad \Delta W_t = W(s_t) - W(s_{t-1})\]
where $s_t$ and $s_{t-1}$ denote two consecutive prefix positions within the same rollout. These step-wise differences isolate the \emph{local change} in reasoning quality and representational drift at each stride, as opposed to the cumulative trends analyzed in Section 2.
\paragraph{Defining mutation events.} We isolate "mutation events"—moments where a trajectory experiences a significant shift in local continuation success. Since the continuation accuracy at each prefix is estimated by completing the reasoning 16 times and recording the fraction of correct completions (Section 2), the minimum detectable change in accuracy is $1/16=0.0625$. We therefore adopt $\Delta |\text{Acc}| \geq 0.0625$ as the threshold for identifying mutation events: changes below this magnitude are indistinguishable from the sampling noise inherent in a 16-trial estimate, while changes at or above this level correspond to at least one completion flipping between correct and incorrect, indicating a meaningful shift in reasoning quality.

As shown in Fig~\ref{fig:local_association}, there is a clear inverse relationship between local accuracy mutations and representational drift (Spearman correlation = $-0.42$). When a trajectory suffers a drop in reasoning quality ($\Delta$Accuracy $< 0$), the W-distance tends to exhibit a corresponding local increase ($\Delta$W-Distance $> 0$). This confirms that Wasserstein distance is a sensitive monitor of local reasoning failures, closely tracking shifts in structural reasoning logic as they unfold.

\begin{figure}[t]
    \centering
    \begin{subfigure}{0.24\linewidth}
        \centering
        \includegraphics[width=1\linewidth]{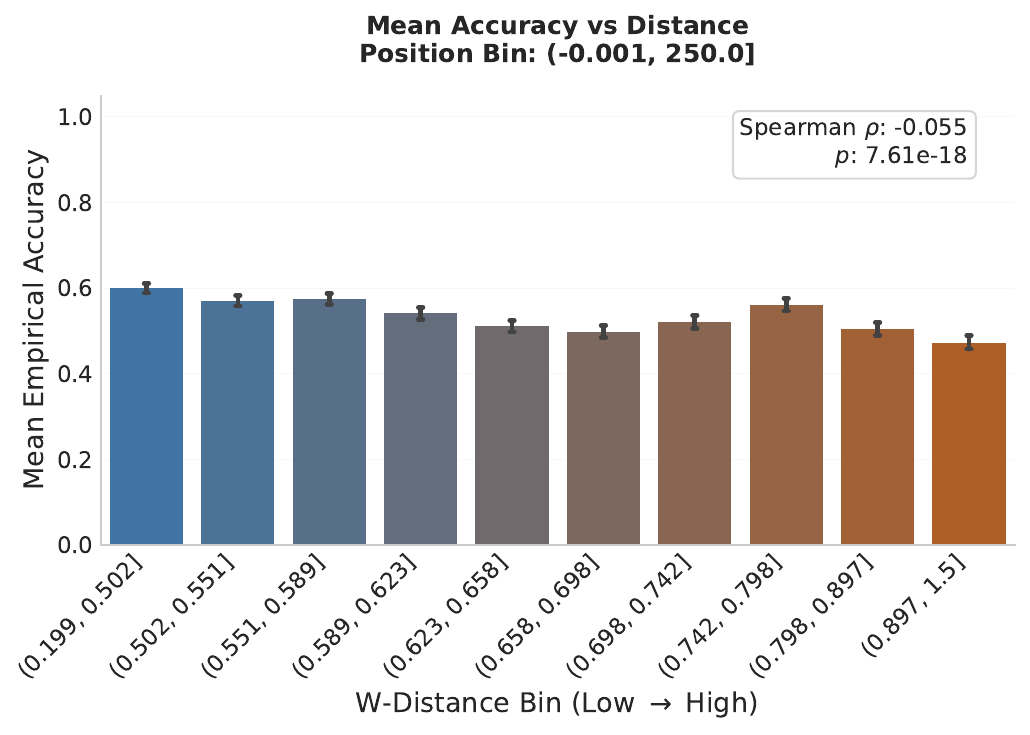}
    \end{subfigure}
    \hfill
    \begin{subfigure}{0.24\linewidth}
        \centering
        \includegraphics[width=1\linewidth]{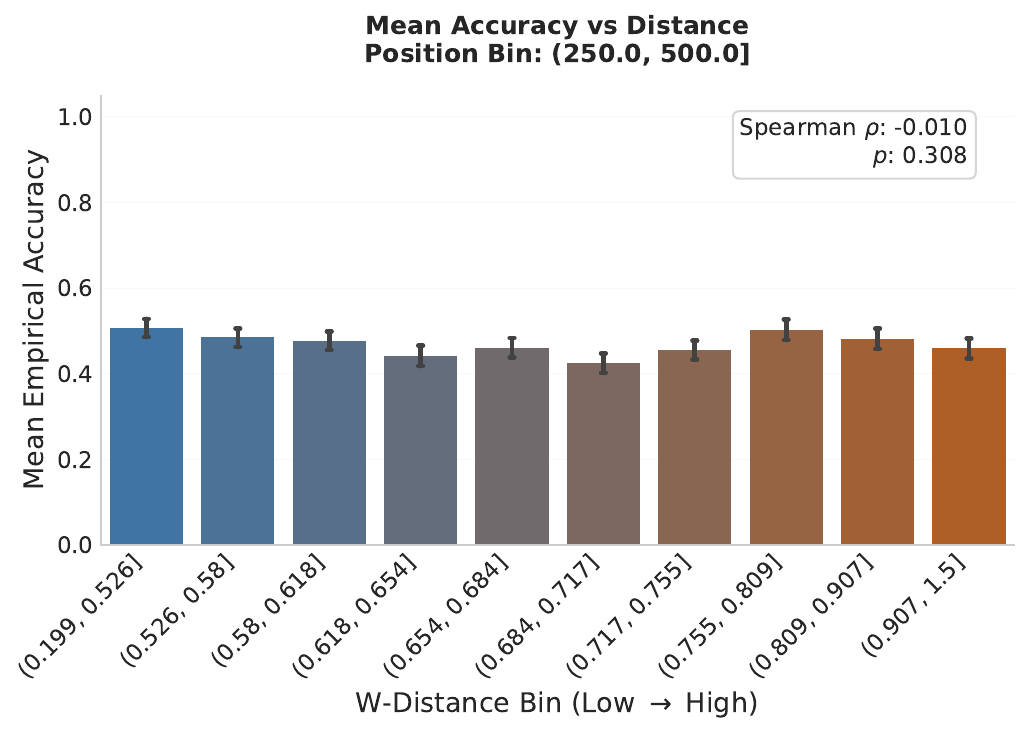}
    \end{subfigure}
    \hfill
    \begin{subfigure}{0.24\linewidth}
        \centering
        \includegraphics[width=1\linewidth]{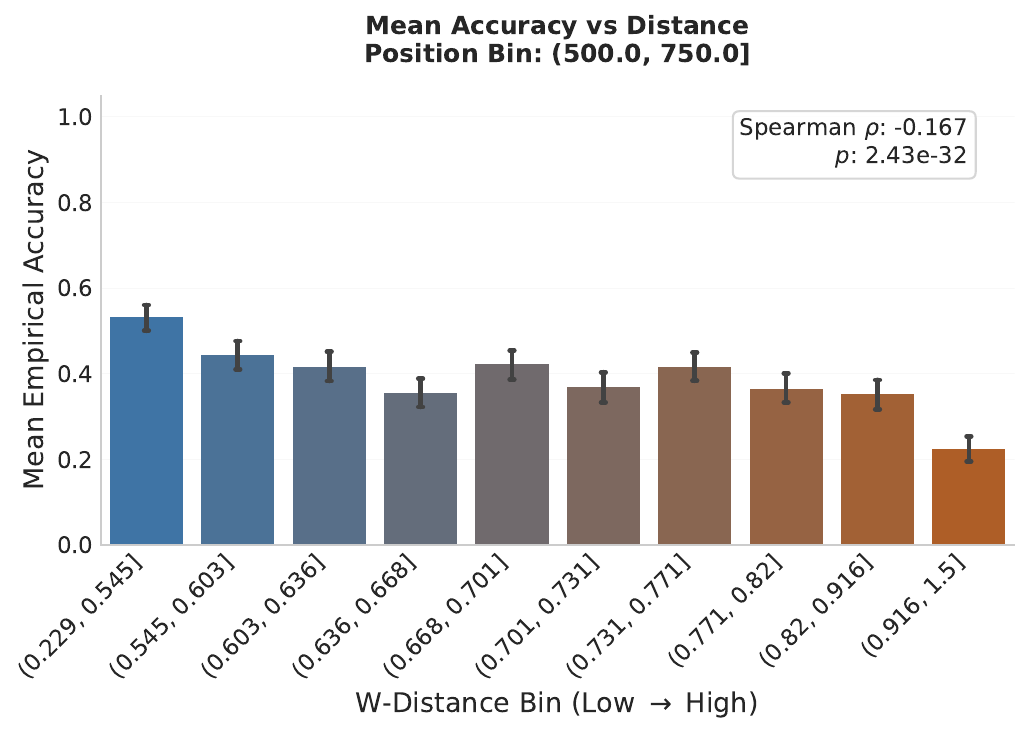}
    \end{subfigure}
    \hfill
    \begin{subfigure}{0.24\linewidth}
        \centering
        \includegraphics[width=1\linewidth]{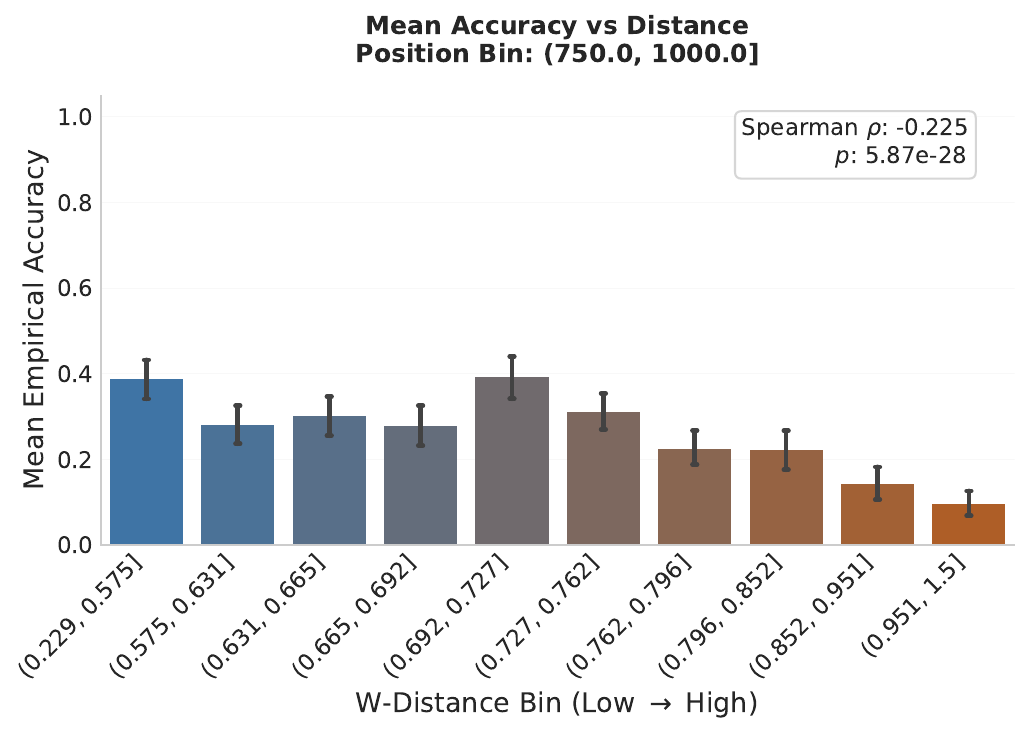}
    \end{subfigure}

    \caption{Wasserstein distance discriminates between robust and vulnerable reasoning trajectories. Samples with low W-distance (left side of each plot) consistently maintain high accuracy, whereas high W-distance (right side) effectively identifies trajectories destined for failure. As the reasoning chain lengthens, the performance gap between low and high W-distance samples widens, and the spearman correlation strengthens from $-0.055$ to $-0.225$, proving that W-distance is a sensitive metric for the internal reliability of the reasoning process.}
    \label{fig:observation_fix_position}
\end{figure}


\section{Sinkhorn Approximation}\label{app:sinkhorn}

The Sinkhorn algorithm~\cite{cuturi2013sinkhorn} computes the entropically regularized optimal transport cost:
\begin{equation}
    W_\epsilon(P, Q) = \min_{\gamma \in \Gamma(P,Q)} \left\{\int \|x - y\|_2 \, d\gamma(x,y) + \epsilon \, \mathrm{KL}(\gamma \| P \otimes Q)\right\}.
\end{equation}

\begin{lemma}[Approximation Bound]\label{lemma:sinkhorn}
For discrete measures on $\leq n$ points within a ball of radius $M$:
\[
W_1(P,Q) \leq W_\epsilon(P,Q) \leq W_1(P,Q) + \epsilon \log n.
\]
\end{lemma}

The Sinkhorn bias $W_\epsilon - W_1 \in [0, \epsilon \log n]$ varies across pairs (depending on the transport plan structure). Therefore, $W_\epsilon$ does not automatically preserve the $W_1$ ordering. However, when the true separation gap exceeds $\epsilon \log n$, the ranking between pre-divergence and post-divergence spans is robust. With our hyperparameters $\epsilon \log n \approx 20.8$, which is negligible relative to the observed separation gap (see Section~\ref{sec:exp:1}).

\section{Complete Proof of Separation Theorem}\label{app:proof-separation}

We restate and fully prove Theorem~\ref{thm:separation}.

\begin{atheorem}[Separation, restated]
Under assumptions (A1)--(A3) stated in Section~\ref{sec:theory}, for a span $S = [a,b]$ of length $n$:
\begin{enumerate}[label=(\roman*)]
    \item If $b \leq \tau$: $\mathbb{E}[W_1(\widehat{P}_S^-, \widehat{P}_S^+)] \leq \eta(n,d)$.
    \item If $a > \tau$: $\mathbb{E}[W_1(\widehat{P}_S^-, \widehat{P}_S^+)] \geq D(S) - \eta(n,d)$.
    \item When $D(S) > 2\eta(n,d)$: $\mathbb{E}[W_1^{\mathrm{post}}] - \mathbb{E}[W_1^{\mathrm{pre}}] > 0$.
\end{enumerate}
\end{atheorem}

\begin{proof}
\textbf{Part (i).} Since $b \leq \tau$, by (A1), for all $t \in [a,b]$, both $h_t^-$ and $h_t^+$ are drawn from $P_t$. Define the population mixture
\[
\bar{P}_S = \frac{1}{n}\sum_{t=a}^{b} P_t.
\]
Both $\widehat{P}_S^-$ and $\widehat{P}_S^+$ are empirical measures formed by $n$ independent samples, where the $t$-th sample in each is drawn from $P_t$. Their common population mixture is $\bar{P}_S = \frac{1}{n}\sum_{t=a}^{b} P_t$. By the triangle inequality:
\begin{equation}
    W_1(\widehat{P}_S^-, \widehat{P}_S^+) \leq W_1(\widehat{P}_S^-, \bar{P}_S) + W_1(\bar{P}_S, \widehat{P}_S^+).
\end{equation}
Taking expectations and applying (A3) to each term:
\begin{equation}
    \mathbb{E}[W_1(\widehat{P}_S^-, \widehat{P}_S^+)] \leq \tilde{C}_d M n^{-1/d} + \tilde{C}_d M n^{-1/d} = \eta(n,d).
\end{equation}

\textbf{Part (ii).} 
Since $a > \tau$, by (A1), $h_t^- \sim Q_t$ and $h_t^+ \sim P_t$ for $t \in [a,b]$. Define population mixtures
\[
\bar{P}_S^- = \frac{1}{n}\sum_{t=a}^{b} Q_t,
\qquad
\bar{P}_S^+ = \frac{1}{n}\sum_{t=a}^{b} P_t,
\]
so $D(S) = W_1(\bar{P}_S^-, \bar{P}_S^+)$.

By the reverse triangle inequality:
\begin{align}
    W_1(\widehat{P}_S^-, \widehat{P}_S^+) &\geq W_1(\bar{P}_S^-, \bar{P}_S^+) - W_1(\widehat{P}_S^-, \bar{P}_S^-) - W_1(\widehat{P}_S^+, \bar{P}_S^+).
\end{align}
Taking expectations and applying (A3):
\begin{align}
    \mathbb{E}[W_1(\widehat{P}_S^-, \widehat{P}_S^+)] &\geq D(S) - \mathbb{E}[W_1(\widehat{P}_S^-, \bar{P}_S^-)] - \mathbb{E}[W_1(\widehat{P}_S^+, \bar{P}_S^+)] \\
    &\geq D(S) - \tilde{C}_d M n^{-1/d} - \tilde{C}_d M n^{-1/d} \\
    &= D(S) - \eta(n,d).
\end{align}

\textbf{Part (iii).} Combining parts (i) and (ii):
\begin{equation}
    \mathbb{E}[W_1^{\mathrm{post}}] - \mathbb{E}[W_1^{\mathrm{pre}}] \geq (D(S) - \eta(n,d)) - \eta(n,d) = D(S) - 2\eta(n,d) > 0
\end{equation}
when $D(S) > 2\eta(n,d)$.
\end{proof}

\subsection{Strict Positivity of $D(S)$}

\begin{aremark}[Non-cancellation]\label{remark:noncancellation}
$D(S) = 0$ would require
\[
\frac{1}{n}\sum_t Q_t = \frac{1}{n}\sum_t P_t
\]
as measures, meaning that distributional differences at individual positions exactly cancel when averaged over the span. This cancellation requires a precise anti-correlation structure across positions in distribution space. We treat $D(S) > 0$ as a regularity condition and verify it empirically: the monotonically increasing AUC between correct and incorrect spans in Section~\ref{sec:exp:1} (Figure~\ref{fig:Empirical_Validation}, right) confirms that post-divergence spans are indeed distributionally separated from their pre-divergence counterparts in practice.
\end{aremark}

\subsection{Concentration Assumption: Justification}\label{app:concentration}

Assumption (A3) states that the empirical Wasserstein convergence rate $O(Mn^{-1/d})$ holds for independent non-identical samples with shared bounded support. The standard proof for the i.i.d.\ case~\cite{fournier2015rate} proceeds via: (i) a covering-number bound on the support, depending only on diameter $2M$ and dimension $d$; (ii) concentration of cell-wise empirical frequencies via Hoeffding's inequality, which applies to sums of independent bounded random variables regardless of whether they are identically distributed. Since both ingredients are agnostic to the heterogeneity of $\{P_t\}$, the rate carries over with a potentially larger constant $\tilde{C}_d \geq C_d$.

Additionally, $W_1(\widehat{P}_n, \bar{P}_n)$ satisfies McDiarmid-type concentration: changing any single $X_t$ affects $\widehat{P}_n$ by at most $2M/n$ in $W_1$, yielding:
\begin{equation}
    \mathbb{P}\left[W_1(\widehat{P}_n, \bar{P}_n) > \mathbb{E}[W_1(\widehat{P}_n, \bar{P}_n)] + u\right] \leq \exp\left(-\frac{nu^2}{2M^2}\right).
\end{equation}

We validate the practical implication of this assumption in 
Section~\ref{sec:exp:1}, where the high discriminability 
(AUC up to 0.99) between correct and incorrect spans confirms 
that finite-sample noise does not overwhelm the distributional signal.

\section{Why Wasserstein Distance?}\label{app:why-w}

\begin{aproposition}[Population Distance Bounds]\label{prop:D-bounds-app}
The population distance $D(S) = W_1(\bar{P}_S^-, \bar{P}_S^+)$ satisfies:

\noindent\textbf{(a) First-moment bound:}
\[
D(S) \geq \left\|\frac{1}{n}\sum_{t=a}^{b}(m_t^- - m_t^+)\right\|_2,
\]
where
\[
m_t^+ = \mathbb{E}_{P_t}[h], \qquad m_t^- = \mathbb{E}_{Q_t}[h].
\]

\noindent\textbf{(b) Shape bound:}
\[
D(S) \geq \left|\frac{1}{n}\sum_{t=a}^{b}\bigl(\mathbb{E}_{Q_t}[\|h-c\|] - \mathbb{E}_{P_t}[\|h-c\|]\bigr)\right|
\]
for any reference point $c$.

\noindent\textbf{(c) Optimal bound:}
\[
D(S) = \sup_{\|f\|_{\mathrm{Lip}} \leq 1} \left|\frac{1}{n}\sum_t (\mathbb{E}_{Q_t}[f] - \mathbb{E}_{P_t}[f])\right|
\]
(KR duality).

\noindent\textbf{(d) Relationship to TV:}
\[
W_1(P,Q) \leq \mathrm{diam}(\mathrm{supp}) \cdot \mathrm{TV}(P,Q)
\]
(upper bound, not useful as a lower bound on $D$).
\end{aproposition}

\begin{proof}
(a) Choose $f(h) = v^\top h$ with $\|v\|_2 = 1$, which is 1-Lipschitz, in the KR dual, then optimize $v$ to align with $\frac{1}{n}\sum_t(m_t^- - m_t^+)$.

(b) $\phi_c(h) = \|h - c\|_2$ is 1-Lipschitz by the reverse triangle inequality. Apply this in the KR dual.

(c) Definition of $W_1$ via KR duality.

(d) Standard result: transport cost per unit of mismatched mass is upper bounded by the support diameter.
\end{proof}

\section{Group-Level Separation}\label{app:group}

The algorithm computes
\[
d_{\min}(S) = \min_{k,j} W_\epsilon(\widehat{P}_S^-, \widehat{P}_{S_j}^{+,k})
\]
over all correct rollouts and their spans. We extend Theorem~\ref{thm:separation} to this setting.

\begin{definition}[Latest Divergence Point]
For incorrect rollout $y^-$ and correct rollouts $\{y^{+,1}, \ldots, y^{+,K^+}\}$ with pairwise divergence points $\tau_1, \ldots, \tau_{K^+}$:
\[
\tau^* = \max_k \tau_k.
\]
\end{definition}

\begin{atheorem}[Group-Level Separation]\label{thm:group-app}
Under Assumptions~(A1--A3):
\begin{enumerate}[label=(\roman*), nosep]
    \item If $b \leq \tau^*$: $\mathbb{E}[d_{\min}(S)] \leq \eta(n,d)$.
    \item If $a > \tau^*$: $\mathbb{E}[d_{\min}(S)] \geq D^*(S) - \eta(n,d)$, where
    \[
    D^*(S) = \min_{k,j} W_1(\bar{P}_S^-, \bar{P}_{S_j}^{+,k}).
    \]
    \item Separation holds when $D^*(S) > 2\eta(n,d)$.
\end{enumerate}
\end{atheorem}

\begin{proof}
\textbf{(i)} Since $b \leq \tau^*$, there exists $y^{+,k^*}$ with $\tau_{k^*} = \tau^*$, meaning that $y^-$ and $y^{+,k^*}$ share the same marginal distributions $\{P_t\}$ for all $t \leq \tau^*$. Since $b \leq \tau^*$ and rollouts in the same GRPO group are generated for the same input, $y^{+,k^*}$ contains a span at position $[a,b]$.\footnote{This requires $y^{+,k^*}$ to be at least $b$ tokens long, which holds in practice since rollouts for the same input have comparable lengths.} Let $S'$ denote this span. Since both $\widehat{P}_S^-$ and $\widehat{P}_{S'}^{+,k^*}$ are empirical measures whose samples are drawn from the same distributions $\{P_t\}_{t=a}^{b}$, we have
\[
d_{\min}(S) \leq W_1(\widehat{P}_S^-, \widehat{P}_{S'}^{+,k^*}),
\]
and the pairwise pre-divergence bound (Theorem~\ref{thm:separation}(i)) gives $\mathbb{E}[d_{\min}(S)] \leq \eta(n,d)$.

\textbf{(ii)} For any $(k, S_j)$, the reverse triangle inequality gives:
\[
W_1(\widehat{P}_S^-, \widehat{P}_{S_j}^{+,k}) \geq W_1(\bar{P}_S^-, \bar{P}_{S_j}^{+,k}) - W_1(\widehat{P}_S^-, \bar{P}_S^-) - W_1(\widehat{P}_{S_j}^{+,k}, \bar{P}_{S_j}^{+,k}).
\]
Since this holds for every $(k,j)$, taking the minimum over $(k,j)$ on both sides:
\[
d_{\min}(S) \geq D^*(S) - W_1(\widehat{P}_S^-, \bar{P}_S^-) - \max_{k,j}\, W_1(\widehat{P}_{S_j}^{+,k}, \bar{P}_{S_j}^{+,k}).
\]
Taking expectations and applying (A3) to each deviation term (noting that (A3) applies to $\{Q_t\}$ under the same bounded-support condition (A2)) yields $\mathbb{E}[d_{\min}(S)] \geq D^*(S) - \eta(n,d)$.
\end{proof}

\begin{aremark}
The condition $D^*(S) > 2\eta$ requires this post-divergence span to be distributionally different from \emph{every} opposing span. If some distant correct span happens to be distributionally close, $D^*$ is small and separation is not guaranteed. This reflects a limitation of the model's representations (inability to distinguish this span from correct reasoning), not of the algorithm.
\end{aremark}

\section{Symmetric Analysis: Correct Rollouts}\label{app:symmetric}

The algorithm applies symmetrically: correct rollouts also compute distances to the opposing (incorrect) group.

\begin{aproposition}[Correct Rollout Separation]\label{prop:correct-app}
For a span $S = [a,b]$ in a correct rollout $y^+$ with divergence point $\tau$ relative to an incorrect rollout $y^-$ and define $d^+(S) := \min_{j} W_1(\widehat{P}_S^+, \widehat{P}_{S_j}^-)$ as the minimum empirical distance from span $S$ in $y^+$ to all spans in $y^-$:
\begin{enumerate}[label=(\roman*), nosep]
    \item If $b \leq \tau$: $\mathbb{E}[d^+(S)] \leq \eta(n,d)$.
    \item If $a > \tau$: $\mathbb{E}[d^+(S)] \geq D^+(S) - \eta(n,d)$, where $D^+(S)$ is the minimum population distance from $S$ to any span in $y^-$.
\end{enumerate}
\end{aproposition}

\begin{proof}
Identical structure to Theorem~\ref{thm:separation}, with the roles of $P_t$ and $Q_t$ exchanged.
\end{proof}

\begin{remark}[Bidirectional Credit Assignment]\label{thm:bilateral-app}
This follows directly from Theorem~\ref{thm:separation} and Proposition~\ref{prop:correct-app} combined with the sign structure of $A^{(i)}$ in Algorithm~\ref{alg:method}.
When separation holds for both incorrect and correct rollouts:
\begin{enumerate}[label=(\alph*), nosep]
    \item Incorrect rollouts ($A^{(i)} < 0$): post-divergence tokens receive higher weights $\omega_t^{(i)}$, increasing the magnitude of negative updates in high-discrepancy regions.
    \item Correct rollouts ($A^{(i)} > 0$): post-divergence tokens receive higher weights $\omega_t^{(i)}$, increasing the magnitude of positive updates in regions more strongly separated from the opposing set.
\end{enumerate}
\end{remark}

\section{Multiple Divergence Points}\label{app:multi-div}

The single-divergence model (Assumption A1) is an analytical simplification. Real reasoning chains may contain multiple divergences or partial self-corrections.

\begin{definition}[Piecewise Divergence]
Rollout $y^-$ has regime boundaries
\[
0 = \tau_0 < \tau_1 < \cdots < \tau_{2m+1} = T
\]
alternating between \emph{aligned regimes} ($h_t^- \sim P_t$) and \emph{drifted regimes} ($h_t^- \sim Q_t^{(j)} \neq P_t$).
\end{definition}

\begin{aproposition}[Piecewise Separation]
Spans fully contained within an aligned regime satisfy the pre-divergence bound. Spans fully contained within a drifted regime satisfy the post-divergence bound with regime-specific $D(S)$. Spans that straddle a regime boundary exhibit intermediate distances proportional to the drifted fraction.
\end{aproposition}

The algorithm operates identically regardless of the number of divergence points: it computes Wasserstein distances for all spans and lets the distances encode the local alignment/drift structure.

\section{Design Choice Justifications}\label{app:design}

\subsection{Max-Pooling vs Mean-Pooling}
For a token $t$ at the divergence boundary (belonging to both pre- and post-divergence spans):
\begin{itemize}[nosep]
    \item Max-pooling assigns $\omega_t = \max_{k:\, t \in S_k} d_k \geq d_{\mathrm{post}}$, preserving the post-divergence signal.
    \item Mean-pooling assigns $\omega_t = \frac{1}{|\{k: t \in S_k\}|}\sum_{k:\, t \in S_k} d_k$, diluted by pre-divergence spans.
\end{itemize}
Max-pooling is the conservative choice: it ensures boundary tokens are weighted by their most discriminative span, which is where credit assignment is most uncertain and impactful.

\subsection{Monotonicity}

Under the natural expectation that distributional divergence accumulates over time (later errors compound earlier ones), $D(S)$ should increase for later post-divergence spans. We do not formally prove this (it requires that the KR dual's optimal witness behaves consistently across span positions), but observe it empirically in Section~5.1.



\subsection{Connection to Process Reward Models}
Let $\mathrm{PRM}(t) = \mathbb{P}[\text{correct} \mid \text{reasoning up to step } t]$. Under the divergence model, $\mathrm{PRM}(t)$ is high for $t \leq \tau$ and decreases for $t > \tau$. The token-level weight $\omega_t^{(i)}$ (Algorithm~\ref{alg:method}, line 13) qualitatively tracks $1 - \mathrm{PRM}(t)$ for incorrect rollouts: small when reasoning is on track, large after errors. This connection is interpretive; the key practical advantage is that $\omega_t^{(i)}$ requires no step-level supervision.

\section{Difficulty Distribution of Retained Subset}\label{app:acc_distribution}
\begin{figure}[ht]
    \centering
    \includegraphics[width=0.5\linewidth]{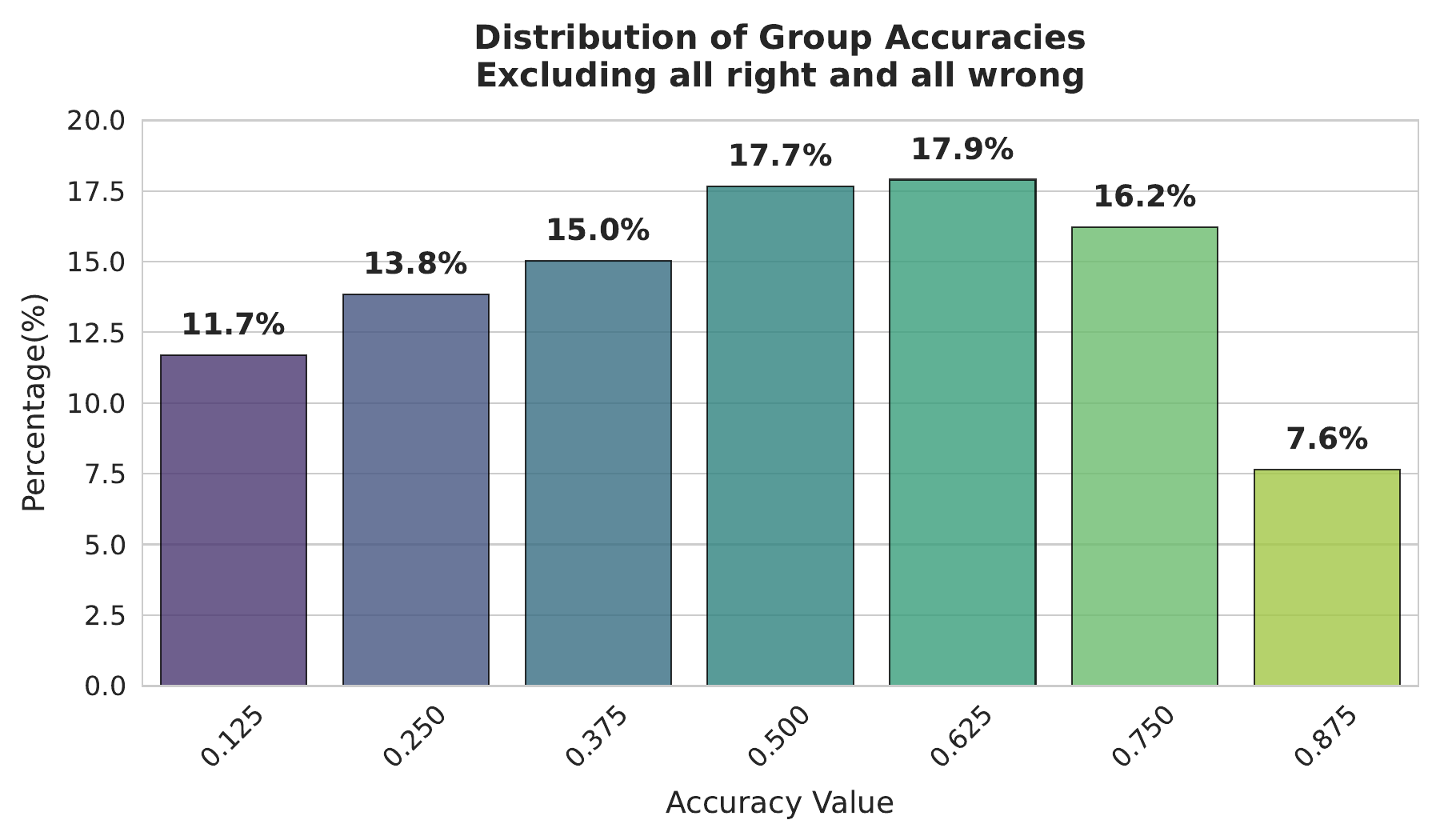}
    \caption{\textbf{Distribution of group accuracies for the retained MATH500 subset (excluding all-correct and all-incorrect groups).} The retained problems span the full range of intermediate difficulty levels, with no single accuracy bin dominating.}
    \label{fig:group_acc_distribution}
\end{figure}
Figure~\ref{fig:group_acc_distribution} shows the group-accuracy distribution of the retained subset. Filtering out all-correct and all-incorrect groups preserves $83.8\%$ of MATH500, indicating that only a small fraction of problems are discarded. The remaining problems are spread relatively evenly across accuracy bins from $0.125$ to $0.875$, with the largest bins ($0.500$ and $0.625$) accounting for less than $18\%$ each. This balanced coverage across difficulty levels suggests that our observations are not driven by a narrow slice of easy or hard problems, but rather reflect general reasoning behavior across the difficulty spectrum.


\section{Implementation Details}\label{sec:setting_details}
The training setup includes a batch size of 512, a learning rate of $1 \times 10^{-6}$, and a clip range between 0.2 and 0.28. Each response sequence is up to 3k tokens in length. The mini-batch size is set to 32. The temperature is 1.0 for training and 0.1 for evaluation. Following prior RLVR work~\citep{liu2025GHPO,yang2025SSPO}, we perform 8 rollouts per prompt and do not use entropy regularization or KL penalties during training (KL coefficient = 0, entropy loss = 0). For math reasoning, We report avg@32 on AIME 2024 and AIME 2025, and avg@8 on AMC 2023, because these benchmarks contain relatively few problems, averaging over multiple samples yields more stable estimates of model capability. For MATH500 and OlympiadBench, which are substantially larger, we report avg@1. For code generation, We report avg@4 across all code benchmarks, which balances the cost of repeated execution-based evaluation against variance reduction. 

\newpage
\section{Case Study}
We present the performance on mathematical data. Figure~\ref{fig:case_heatmap} shows the pair-wise wasserstein ditance matrix.

\noindent
\renewcommand{\arraystretch}{1.15}
\begin{tabular}{|>{\raggedright\arraybackslash}p{0.46\textwidth}|>{\raggedright\arraybackslash}p{0.46\textwidth}|}
\hline
\multicolumn{2}{|>{\centering\arraybackslash}p{0.93\textwidth}|}{%
\textcolor{black}{\textbf{Question:}} \textcolor{black}{How many ordered four-tuples of integers $(a, b, c, d)$ with $0 < a < b < c < d < 500$ satisfy:
\textbf{(1)} $a + d = b + c$ \quad \textbf{(2)} $bc - ad = 93$ \quad\quad \newline Remember to output the final answer in the format of \textbackslash\textbackslash boxed\{\}}
} \\
\hline
\rowcolor{lightgray}
\textbf{Case 1} \hfill \textcolor{correctgreen}{\textbf{$\boxed{870}$\, \checkmark}} &
\textbf{Case 2} \hfill \textcolor{incorrectred}{\textbf{$\boxed{2450}$\, $\times$}} \\
\hline
 
{\small
To find the number of ordered four-tuples $ (a, b, c, d) $ with $ 0 < a < b < c < d < 500 $ that satisfy the conditions:\newline
1. $ a + d = b + c $\newline
2. $ bc - ad = 93 $

We start by analyzing the given equations. From the condition $ a + d = b + c $, we can express $ d $ in terms of $ a, b, c $:\newline
$ d = b + c - a $

Substitute $ d = b + c - a $ into the second equation $ bc - ad = 93 $:\newline
$ bc - a(b + c - a) = 93 $\newline
$ bc - ab - ac + a^2 = 93 $\newline
$ b(c - a) - a(c - a) = 93 $\newline
$ (b - a)(c - a) = 93 $

Next, we need to find all pairs of integers $ (m, n) $ such that $ m = b - a $ and $ n = c - a $ and $ m \cdot n = 93 $. The integer factor pairs of 93 are:\newline
$ (1, 93), (3, 31) $

We iterate through these pairs and find eligible values of $ a, b, c, d $ satisfying $ 0 < a < b < c < d < 500 $.

\subsubsection*{Case 1: $ m \cdot n = 93 $ with $ m = b - a $ and $ n = c - a $}
For each pair $ (m, n) $:

\subsubsection*{Subcase 1: $ m = 1 $ and $ n = 93 $}
$ b - a = 1 $ and $ c - a = 93 $\newline
Thus, $ b = a + 1 $ and $ c = a + 93 $.
$ d = b + c - a = (a + 1) + (a + 93) - a = a + 94 $\newline
We require $ 0 < a < b < c < d < 500 $, hence:
$ 0 < a < a + 1 < a + 93 < a + 94 < 500 $\newline
This simplifies to:
$ 0 < a < 406 $\newline
Since $ a $ must be an integer:
$ 1 \leq a \leq 405 $\newline
There are 405 possible values for $ a $.

\subsubsection*{Subcase 2: $ m = 3 $ and $ n = 31 $}
$ b - a = 3 $ and $ c - a = 31 $\newline
Thus, $ b = a + 3 $ and $ c = a + 31 $.
$ d = b + c - a = (a + 3) + (a + 31) - a = a + 34 $\newline
We require $ 0 < a < b < c < d < 500 $, hence:\newline
$ 0 < a < a + 3 < a + 31 < a + 34 < 500 $\newline
This simplifies to:
$ 0 < a < 466 $\newline
Since $ a $ must be an integer:
$ 1 \leq a \leq 465 $\newline
There are 465 possible values for $ a $.

\subsubsection*{Total number of solutions:}
- For the pair $ (1, 93) $, there are 405 possible values of $ a $.\newline
- For the pair $ (3, 31) $, there are 465 possible values of $ a $.

Since these two cases are mutually exclusive, the total number of solutions is:
$ 405 + 465 = 870 $

So, the final answer is:
$ \boxed{870} $
}
 
&

{\small
To find the number of ordered four-tuples $(a, b, c, d)$ with $0 < a < b < c < d < 500$ that satisfy the equations $a + d = b + c$ and $bc - ad = 93$, we proceed with the following steps:

\subsubsection*{Step 1: Understand the given equations}
We have:\newline
1. $a + d = b + c$\newline 
2. $bc - ad = 93$

\subsubsection*{Step 2: Express $d$ in terms of $a, b,$ and $c$ from the first equation}
From $a + d = b + c$, we get:
$d = b + c - a$

\subsubsection*{Step 3: Substitute $d$ into the second equation}

Substitute $d = b + c - a$ into $bc - ad = 93$:
$bc - a(b + c - a) = 93$
$bc - ab - ac + a^2 = 93$
$bc - ab - ac + a^2 = 93$
$bc - ab - ac + a^2 = 93$
$b c - a b - a c + a^2 = 93$
$b c - a b - a c + a^2 = 93$
$b c - a b - a c + a^2 = 93$
$b (c - a) - a c + a^2 = 93$
$b (c - a) + a^2 - a c = 93$
$b (c - a) + a^2 - a c = 93$

\subsubsection*{Step 4: Iterate over possible values for $a, b, c$ and find $d$}
We need $0 < a < b < c < d < 500$ and $a + d = b + c$ and $bc - ad = 93$.

We will write a program to iterate over all possible values of $a, b, c$ and compute $d$ to check if both conditions are met.

Let’s implement this approach:
\begin{verbatim}
count = 0
for a in range(1, 499):
  for b in range(a+1, 500):
    for c in range(b+1, 500):
      d = b + c - a
      if d > c and d < 500:
        if b*c - a*d == 93:
          count += 1
print(count)
\end{verbatim}
 
Output: \texttt{2450}\newline
The number of ordered four-tuples $(a, b, c, d)$ with $0 < a < b < c < d < 500$ that satisfy the conditions $a + d = b + c$ and $bc - ad = 93$ is $\boxed{2450}$.
}
 
\\
\hline
\end{tabular}

 
\noindent
\renewcommand{\arraystretch}{1.15}
\begin{tabular}{|>{\raggedright\arraybackslash}p{0.46\textwidth}|>{\raggedright\arraybackslash}p{0.46\textwidth}|}
\hline
\rowcolor{lightgray}
\textbf{Case 1} \hfill \textcolor{correctgreen}{\textbf{$\boxed{870}$\, \checkmark}} &
\textbf{Case 3} \hfill \textcolor{correctgreen}{\textbf{$\boxed{870}$\, \checkmark}} \\
\hline
 
{\small
To find the number of ordered four-tuples $ (a, b, c, d) $ with $ 0 < a < b < c < d < 500 $ that satisfy the conditions:\newline
1. $ a + d = b + c $\newline
2. $ bc - ad = 93 $

We start by analyzing the given equations. From the condition $ a + d = b + c $, we can express $ d $ in terms of $ a, b, c $:\newline
$ d = b + c - a $

Substitute $ d = b + c - a $ into the second equation $ bc - ad = 93 $:\newline
$ bc - a(b + c - a) = 93 $\newline
$ bc - ab - ac + a^2 = 93 $\newline
$ b(c - a) - a(c - a) = 93 $\newline
$ (b - a)(c - a) = 93 $

Next, we need to find all pairs of integers $ (m, n) $ such that $ m = b - a $ and $ n = c - a $ and $ m \cdot n = 93 $. The integer factor pairs of 93 are:\newline
$ (1, 93), (3, 31) $

We iterate through these pairs and find eligible values of $ a, b, c, d $ satisfying $ 0 < a < b < c < d < 500 $.

\subsubsection*{Case 1: $ m \cdot n = 93 $ with $ m = b - a $ and $ n = c - a $}
For each pair $ (m, n) $:

\subsubsection*{Subcase 1: $ m = 1 $ and $ n = 93 $}
$ b - a = 1 $ and $ c - a = 93 $\newline
Thus, $ b = a + 1 $ and $ c = a + 93 $.
$ d = b + c - a = (a + 1) + (a + 93) - a = a + 94 $\newline
We require $ 0 < a < b < c < d < 500 $, hence:
$ 0 < a < a + 1 < a + 93 < a + 94 < 500 $\newline
This simplifies to:
$ 0 < a < 406 $\newline
Since $ a $ must be an integer:
$ 1 \leq a \leq 405 $\newline
There are 405 possible values for $ a $.

\subsubsection*{Subcase 2: $ m = 3 $ and $ n = 31 $}
$ b - a = 3 $ and $ c - a = 31 $\newline
Thus, $ b = a + 3 $ and $ c = a + 31 $.
$ d = b + c - a = (a + 3) + (a + 31) - a = a + 34 $\newline
We require $ 0 < a < b < c < d < 500 $, hence:\newline
$ 0 < a < a + 3 < a + 31 < a + 34 < 500 $\newline
This simplifies to:
$ 0 < a < 466 $\newline
Since $ a $ must be an integer:
$ 1 \leq a \leq 465 $\newline
There are 465 possible values for $ a $.

\subsubsection*{Total number of solutions:}
- For the pair $ (1, 93) $, there are 405 possible values of $ a $.\newline
- For the pair $ (3, 31) $, there are 465 possible values of $ a $.

Since these two cases are mutually exclusive, the total number of solutions is:
$ 405 + 465 = 870 $

So, the final answer is:
$ \boxed{870} $
}
 
&

{\small
To solve the problem, we need to find the number of ordered four-tuples $(a, b, c, d)$ of integers that satisfy the following conditions:\newline
$0 < a < b < c < d < 500$\newline
$a + d = b + c$\newline
$bc - ad = 93$

First, let's rewrite the equation $a + d = b + c$ as $d = b + c - a$.

Substitute $d = b + c - a$ into the second equation $bc - ad = 93$:\newline
$bc - a(b + c - a) = 93$\newline
$bc - ab - ac + a^2 = 93$\newline
$b(c - a) - a(c - a) = 93$\newline
$(b - a)(c - a) = 93$

We need to find pairs of integers $(x, y)$ such that $x = b - a$ and $y = c - a$, and $xy = 93$. The integer factor pairs of 93 are:\newline
$(1, 93), (3, 31)$

Since $0 < a < b < c < d < 500$, we ensure $a, b, c, d$ are valid integers in the range from 1 to 499.

\subsubsection*{Case 1: $x = 1$ and $y = 93$}
$b - a = 1 \quad \text{and} \quad c - a = 93$
This implies $b = a + 1$ and $c = a + 93$.

Next, substitute $d = b + c - a$:
$d = (a + 1) + (a + 93) - a = a + 94$

We need $0 < a < b < c < d < 500$:
$0 < a < (a + 1) < (a + 93) < (a + 94) < 500$

This translates to:
$0 < a$
$a + 94 < 500 \implies a < 406$

So, $1 \le a \le 405$ (since $a$ must be a positive integer).
The number of valid $a$ is 405.

\subsubsection*{Case 2: $x = 3$ and $y = 31$}
$b - a = 3 \quad \text{and} \quad c - a = 31$\newline
This implies $b = a + 3$ and $c = a + 31$.

Next, substitute $d = b + c - a$:\newline
$d = (a + 3) + (a + 31) - a = a + 34$

We need $0 < a < b < c < d < 500$:\newline
$0 < a < (a + 3) < (a + 31) < (a + 34) < 500$

This translates to:\newline
$0 < a$\newline
$a + 34 < 500 \implies a < 466$

So, $1 \le a \le 465$.\newline
The number of valid $a$ is 465.

\subsubsection*{Total number of solutions}
We have two cases:
\begin{enumerate}
    \item From $x = 1$ and $y = 93$: $405$ valid four-tuples.
    \item From $x = 3$ and $y = 31$: $465$ valid four-tuples.
\end{enumerate}

There is no overlap between the two cases because the factor pairs are different. So the total number of valid four-tuples is:
$405 + 465 = 870$

Thus, the number of ordered four-tuples $(a, b, c, d)$ that satisfy the conditions is $\boxed{870}$.
}
 
\\
\hline
\end{tabular}

\newpage
\begin{figure}[t]
\centering
\begin{subfigure}[b]{0.48\textwidth}
    \includegraphics[width=\textwidth]{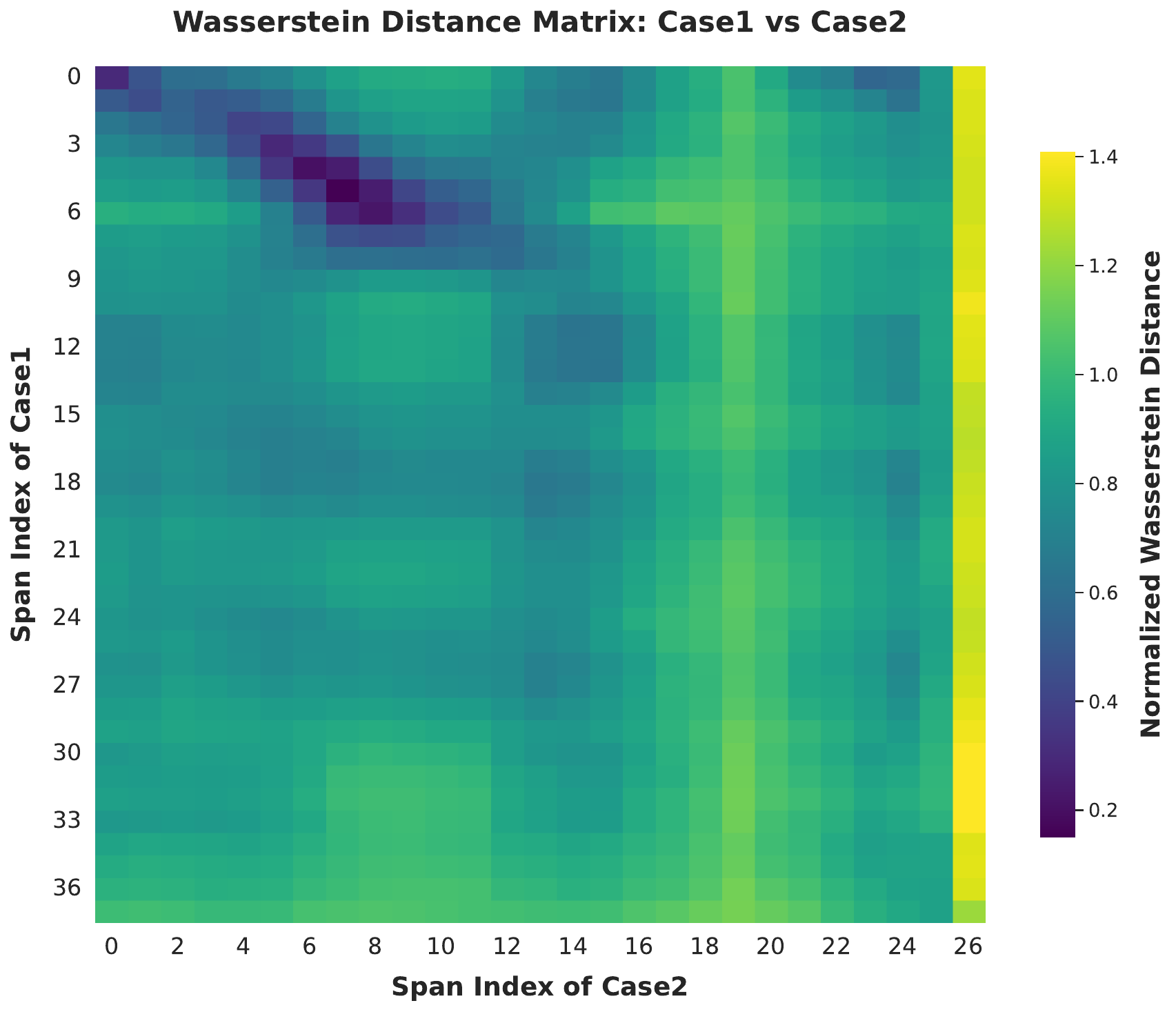}
    \label{fig:case1_case2}
\end{subfigure}
\hfill
\begin{subfigure}[b]{0.48\textwidth}
    \includegraphics[width=\textwidth]{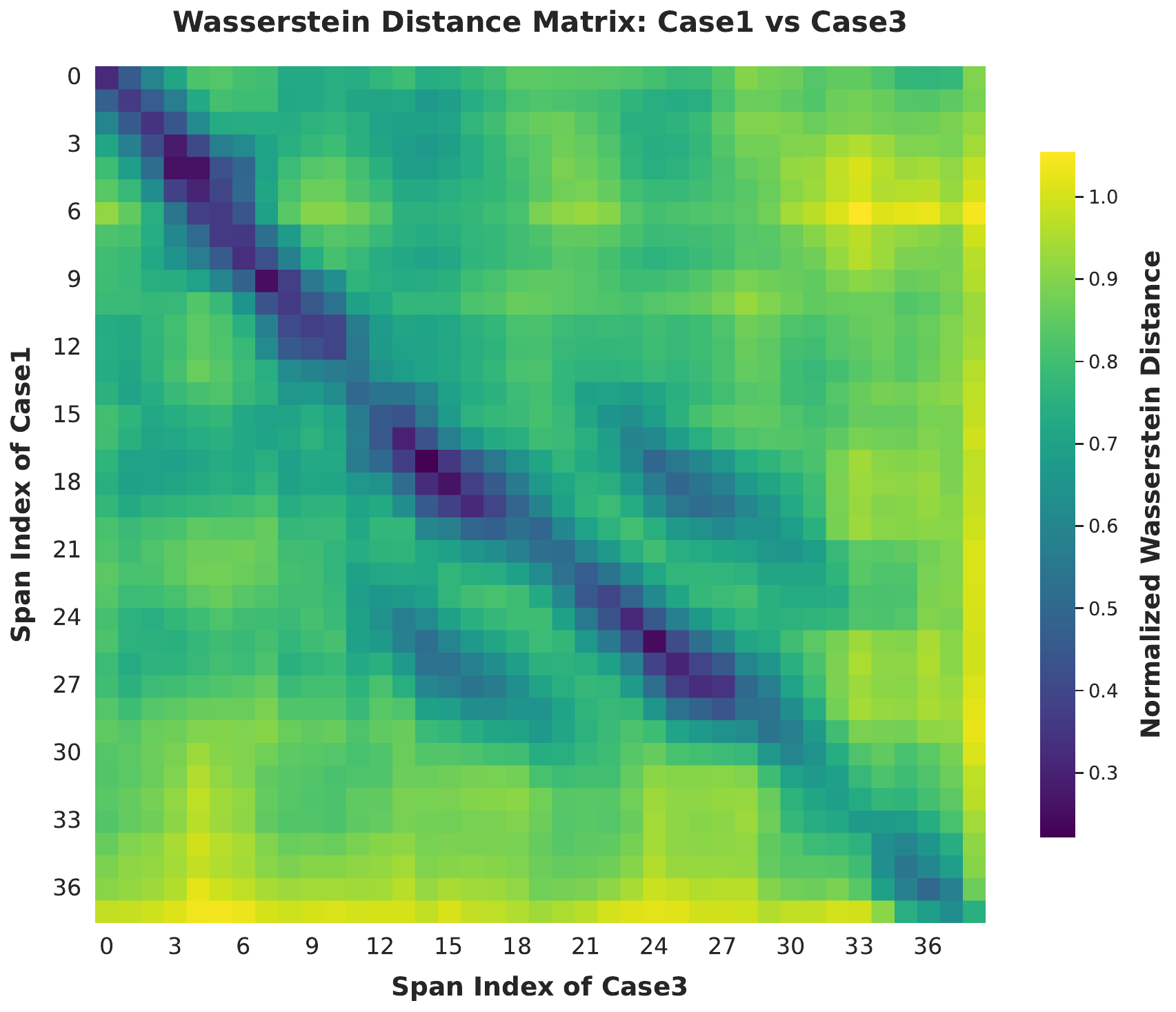}
    \label{fig:case1_case3}
\end{subfigure}
\caption{\textbf{Normalized Wasserstein Distance Heatmap of Case.}}
\label{fig:case_heatmap}
\end{figure}

\newpage

\bibliography{references}{}
\bibliographystyle{plain}

\end{document}